%% file: main.tex
\documentclass[final]{cvpr}

\input{setup/package}

\input{setup/macros}

\input{setup/math.tex}

\def\cvprPaperID{007}
\def\confYear{CVPR 2021}

\begin{document}
\definecolor{mygray}{gray}{0.4}
\newenvironment{mytiny}{\color{mygray}\tiny}{}

\title{Regularizing Generative Adversarial Networks under Limited Data}

\author{{Hung-Yu Tseng\thanks{}\hspace{5pt}}$^2$,
Lu Jiang$^1$,
Ce Liu$^1$,
Ming-Hsuan Yang$^{1,2,4}$,\\
Weilong Yang$^3$ \vspace{1.5mm}\\
$^{1}$Google Research\hspace{20pt}$^{2}$University of California, Merced\hspace{20pt}$^{3}$Waymo\hspace{20pt}$^{4}$Yonsei University
}

\twocolumn[{
\renewcommand\twocolumn[1][]{#1}
\maketitle
\begin{center}
    \centering
    \includegraphics[width=\linewidth]{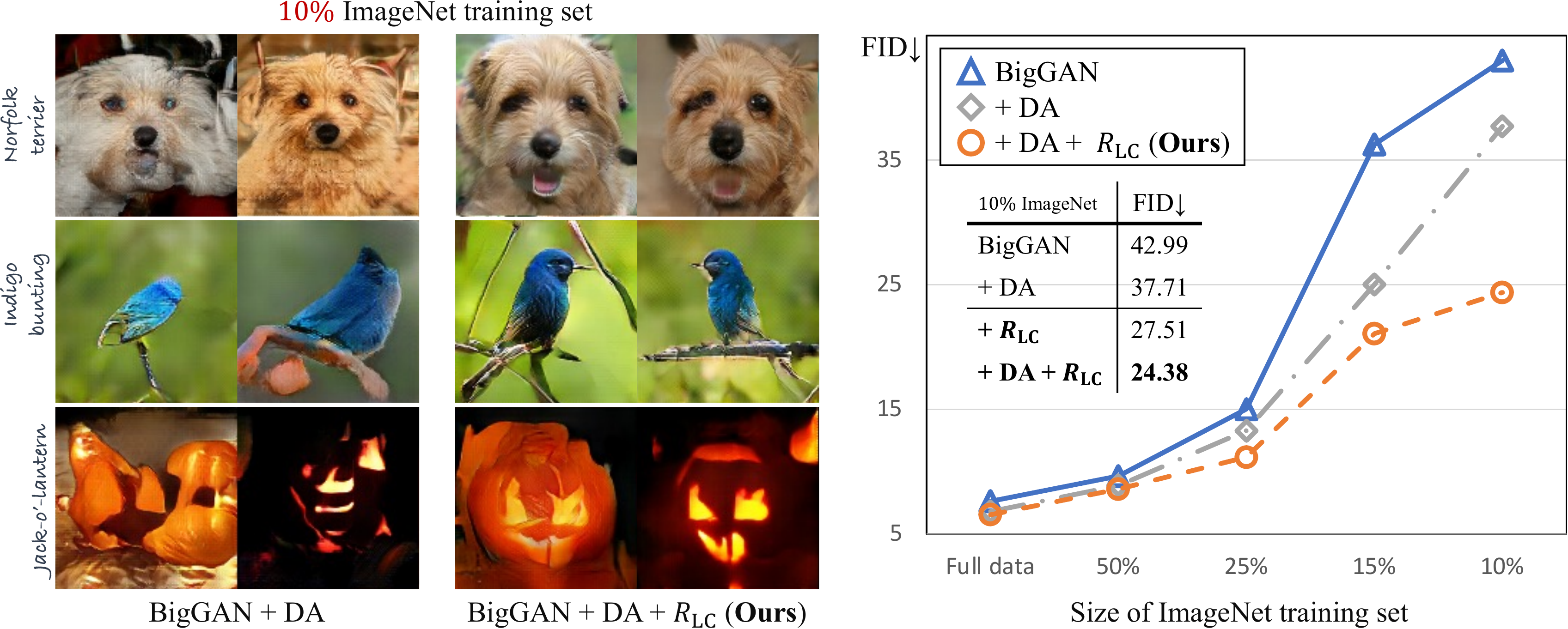}
    \captionof{figure}{\textbf{Regularizing GANs under limited training data.}  (\textit{left}) Image generation trained on $10\%$ ImageNet training set; (\textit{right}) FID scores \vs ImageNet training set size. The proposed regularization method 1) addresses the limited training data issue for the GAN models, and 2) is empirically complementary to the recent data augmentation approaches~\cite{ada,diffaug}.
    }
    \vspace{2.5mm}
    \label{fig:teaser}
\end{center}
}]
{
  \renewcommand{\thefootnote}%
    {\fnsymbol{footnote}}
  \footnotetext[1]{Work done during HY's internship at Google Research.}
}

\input{0_abstract}
\input{1_introduction}
\input{2_relatedwork}
\input{3_method}
\input{4_result}
\input{5_conclusion}

{\small
\bibliographystyle{ieee_fullname}
\bibliography{egbib}
\balance
}

\clearpage
\onecolumn
\appendix
\input{6_appendix}

\end{document}

%% file: setup/package.tex
\usepackage{times}

\usepackage{epsfig}
\usepackage{graphicx}
\usepackage{subcaption}

\usepackage[lined,ruled,linesnumbered]{algorithm2e}

\usepackage{booktabs}                   
\usepackage{multirow}
\usepackage{makecell}
\usepackage{paralist}
\usepackage{enumitem}

\usepackage{units}

\usepackage{comment}

\usepackage{url}
\usepackage[pagebackref=true,breaklinks=true,colorlinks,bookmarks=false,linkcolor=blue,citecolor=blue]{hyperref}

\usepackage{amsmath}
\usepackage{amsthm}
\usepackage{amssymb}

\usepackage{balance}

%% file: setup/macros.tex
\newtheorem{proposition}{Proposition}
\newcommand{\FFHQ}{FFHQ}

\newcommand{\best}[1]{\mathbf{#1}}
\newcommand{\second}[1]{\underline{#1}}

\def\etal{et~al.}			  
\def\eg{e.g.,~}               
\def\ie{i.e.,~}               
\def\etc{etc}                 
\def\cf{cf.~}                 
\def\vs{vs.~}                 

\DeclareMathOperator*{\argmin}{\arg\!\min}

\newlength\paramargin
\newlength\figmargin
\newlength\secmargin
\newlength\figcapmargin
\newlength\tabcapmargin
\newlength\tabbotmargin

\newcommand{\Paragraph}[1]
{\vspace{1mm} \noindent\textbf{#1}}

\newcommand{\secref}[1]{Section~\ref{sec:#1}}
\newcommand{\figref}[1]{Figure~\ref{fig:#1}} 
\newcommand{\tabref}[1]{Table~\ref{tab:#1}}

\long\def\ignorethis#1{}

%% file: setup/math.tex
\usepackage{amsmath,amsfonts,bm}

\def\argmin{\mathop{\rm argmin}}%

\def\bx{\bm{x}} %

\def\reg{R_{\mathrm{LC}}}

\usepackage{amsmath,amsfonts}
\usepackage{bbm}

\newcommand{\eat}[1]{} 

\def\pd{p_d(\bx)} %
\def\pg{p_g(\bx)} %
\def\pz{p_z(\bx)} %
\def\Dx{D(\bx)} %
\def\Dsx{D^*(\bx)} %

\def\E{\mathop{\mathbb{E}}} %

\def\D{\mathcal{T}} %

\def\bx{\bm{x}} %

\def\bz{\bm{z}} %

%% file: 0_abstract.tex
\begin{abstract}
   Recent years have witnessed the rapid progress of generative adversarial networks~(GANs).
   However, the success of the GAN models hinges on a large amount of training data.
   This work proposes a regularization approach for training robust GAN models on limited data.
   We theoretically show a connection between the regularized loss and an $f$-divergence called LeCam-divergence, which we find is more robust under limited training data.
   Extensive experiments on several benchmark datasets demonstrate that the proposed regularization scheme 1) improves the generalization performance and stabilizes the learning dynamics of GAN models under limited training data, and 2) complements the recent data augmentation methods.
   These properties facilitate training GAN models to achieve state-of-the-art performance when only limited training data of the ImageNet benchmark is available. 
   The source code is available at \url{https://github.com/google/lecam-gan}.
\end{abstract}

%% file: 1_introduction.tex
\section{Introduction}

Generative adversarial networks~(GANs)~\cite{wgan,biggan,goodfellow2014generative,lsgan} have made significant progress in recent years on synthesizing high-fidelity images.
The GAN models are the cornerstone techniques for numerous vision applications, such as data augmentation~\cite{frid2018gan,frid2018synthetic}, domain adaptation~\cite{hoffman2018cycada,hsu2020progressive}, image extrapolation~\cite{InOut,teterwak2019boundless}, image-to-image translation~\cite{huang2018multimodal,lee2020drit++,zhu2017unpaired}, and visual content creation~\cite{abdal2020image2stylegan++,ganpaint,huang2020semantic,huh2020transforming,tseng2020modeling,tseng2020retrievegan}.

The success of the GAN methods heavily relies on a large amount of diverse training data which is often labor-expensive or cumbersome to collect~\cite{webster2019detecting}.
As the example of the BigGAN~\cite{biggan} model presented in \figref{teaser}, the performance significantly deteriorates under the limited training data. 
Consequently, several very recent approaches~\cite{ada,diffaug,zhao2020image} have been developed to address the data insufficiency issue.
A representative task in this emerging research direction aims to learn a robust class-conditional GAN model when only a small proportion of the ImageNet data~\cite{imagenet} are available for the training.
Generally, existing methods exploit data augmentation, either conventional or differentiable augmentation, to increase the diversity of the limited training data.
These data augmentation approaches have shown promising results on several standard benchmarks.

In this paper, we address the GAN training task on \emph{limited data} from a different perspective: model regularization. 
Although there are numerous regularization techniques for the GAN models in the literature~\cite{wgangp,gp0,spectral,sonderby2016amortised,zhou2018don}, none of them aim to improve the generalization of the GAN models trained on limited data.
In contrast, our goal is to \emph{learn robust GAN models on limited training data that can generalize well on out-of-sample data}.
To this end, we introduce a novel regularization scheme to modulate the discriminator's prediction for learning a robust GAN model. 
Specifically, we impose an $\ell_2$ norm between the current prediction of the real image and a moving average variable that tracks the historical predictions of the generated image, and vice versa. 
We theoretically show that, under mild assumptions, the regularization transforms the WGAN~\cite{wgan} formulation towards minimizing an $f$-divergence called LeCam-divergence~\cite{le2012asymptotic}.
We find that the LeCam-divergence is more robust under the limited training data setting.

We conduct extensive experiments to demonstrate the three merits of the proposed regularization scheme. 
First, it improves the generalization performance of various GAN approaches, such as BigGAN~\cite{biggan} and StyleGAN2~\cite{stylegan2}.
Second, it stabilizes the training dynamics of the GAN models under the limited training data setting. 
Finally, our regularization approach is empirically complementary to the data augmentation methods~\cite{ada,diffaug}. 
As presented in \figref{teaser}, we obtain state-of-the-art performance on the limited (\eg $10\%$) ImageNet dataset by combining our regularization (\ie $\reg$) and the data augment method~\cite{diffaug}.

%% file: 2_relatedwork.tex
\section{Related Work}
\label{sec:related}

\Paragraph{Generative adversarial networks.}
Generative adversarial networks~(GANs)~\cite{wgan,biggan,goodfellow2014generative,relativisticgan,stylegan2,lsgan,zhang2019sagan} aim to model the target distribution using adversarial learning.
Various adversarial losses have been proposed to stabilize the training or improve the convergence of the GAN models, mainly based on the idea of minimizing the $f$-divergence between the real and generated data distributions~\cite{nowozin2016f}.
For example, Goodfellow~\etal~\cite{goodfellow2014generative} propose the saturated loss that minimizes the JS-divergence between the two distributions.
Similarly, the LSGAN~\cite{lsgan} formulation leads to minimizing the $\chi^2$-divergence~\cite{pearson1900x}, and the EBGAN~\cite{zhao2016energy} approach optimizes the total variation distance~\cite{wgan}.
On the other hand, some models are designed to minimize the integral probability metrics (IPM)~\cite{muller1997integral,song2019bridging}, such as the WGAN~\cite{wgan,wgangp} frameworks.
In this work, we design a new regularization scheme that can be applied to different GAN loss functions for training the GAN models on the limited data.

\Paragraph{Learning GANs on limited training data.}
With the objective of reducing the data collection effort, several studies~\cite{gulrajani2020towards,webster2019detecting} raise the concern of insufficient data for training the GAN models.
Training the GAN models on limited data is challenging because the data scarcity leads to the problems such as unstable training dynamics, degraded fidelity of the generated images, and memorization of the training examples.
To address these issues, recent methods~\cite{ada,tran2020towards,zhang2019consistency,diffaug,zhao2020bcr,zhao2020image} exploit data augmentation as a mean to increase data diversity, hence preventing the GAN models from overfitting the training data.
For example, Zhang~\etal~\cite{zhang2019consistency} augment the real images and introduce a consistency loss for training the discriminator.
The DA~\cite{diffaug} and ADA~\cite{ada} approaches share a similar idea of applying differential data augmentation on both real and generated images, in which ADA further develops an adaptive strategy to adjust the probability of augmentation.
In contrast to prior work, we tackle this problem from a different perspective of \emph{model regularization}.
We show that our method is conceptually and empirically complementary to the existing data augmentation approaches.

\Paragraph{Regularization for GANs.}
Most existing regularization methods for GAN models aim to accomplish two goals: 1) stabilizing the training to ensure the convergence~\cite{gp0,mescheder2017numerics}, and 2) mitigating the mode-collapse issue~\cite{salimans2016improved}.
As the GAN frameworks are known for unstable training dynamics, numerous efforts have been made to address the issue using noise~\cite{jenni2019stabilizing,sonderby2016amortised}, gradient penalty~\cite{wgangp,kodali2017convergence,gp0,roth2017stabilizing}, spectral normalization~\cite{spectral}, adversarial defense~\cite{zhou2018don}, \etc.
On the other hand, a variety of regularization approaches~\cite{berthelot2017began,che2016mode,liu2019normalized,mao2019mode,srivastava2017veegan,yang2019diversity} are proposed to alleviate the model-collapse issue, thus increasing the diversity of the generated images.
Compared with these methods, our work targets a different goal: improving the generalization of the GAN models trained on the \emph{limited training data}.

\Paragraph{Robust Deep Learning.}
Robust deep learning aims to prevent the deep neural networks from overfitting or memorizing the training data.
Recent methods have shown successes in overcoming training data bias such as label noise~\cite{arpit2017closer, han2018co, jiang2020beyond, jiang2018mentornet, liu2020early, northcutt2019confident, ren2018learning, xu2021faster} and biased data distributions~\cite{cheng2020advaug, cui2019class, liang2020simaug, lin2017focal, shu2019meta, tseng2020regularizing, tseng2020cross}.
Recently, few approaches~\cite{bora2018ambientgan,kaneko2020noise,kaneko2019label,thekumparampil2018robustness} have been proposed for learning the robust GAN model. 
While these approaches are
designed to overcome label or image noise in a corrupted training set, we improve the generalization of the GAN models trained on the limited \emph{uncorrupted} training data.

%% file: 3_method.tex
\begin{figure*}[t]
    \centering
    \includegraphics[width=0.8\linewidth]{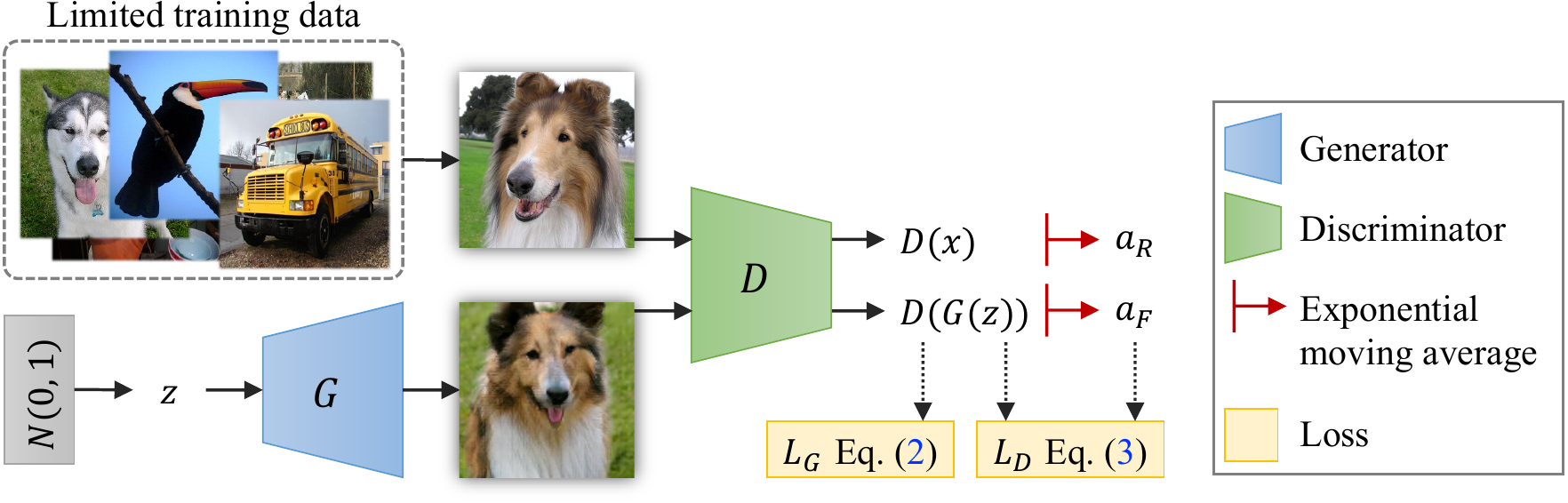}
    \vspace{\figcapmargin}
    \caption{\textbf{Algorithmic overview.} During the GAN training stage, we use the exponential moving average variables, called \emph{anchors}, to track the discriminator predictions. The anchors are then used to compute the regularized discriminator loss described in Eq.~\eqref{eq:loss_d} to improve the generalization performance of the GAN models.
    }
    \vspace{-2mm}
    \label{fig:overview}
\end{figure*}

\section{Methodology}
\label{sec:3}
We first review the GAN models, then detail our regularization scheme.
Finally, we discuss the connection between the proposed method and the LeCam-divergence along with the effect on robust learning under the limited data setting.

\subsection{Generative Adversarial Networks}
\label{sec:gan}
A GAN model consists of a discriminator $D$ and a generator $G$.
Let $V_D$ and $L_G$ denote the training objectives of the discriminator $D$ and generator $G$, respectively.
The training of the GAN frameworks can be generally illustrated as:
{\small
\begin{align}
\max_D V_D, V_D \!=\!& \E_{\bx \sim \D} \big[ f_D(D(\bx)) \big]
\!+\! \E_{\bz \sim p_{\bz}} \big[ f_G(D(G(\bz))) \big]\\
\min_G L_G, L_G \!=\!& \E_{\bz \sim p_{\bz}} \big[ g_G(D(G(\bz))) \big],
\label{eq:gan}
\vspace{-1.5mm}
\end{align}
}\noindent
where $p_\mathbf{z}$ is the prior distribution~(\eg $\mathcal{N}(0, I)$) and $\D$ is the training (observed) image set used to approximate the data distribution. 
The notations $f_D$, $f_G$, and $g_G$ in Eq.~\eqref{eq:gan} represent the mapping functions from which various GAN losses can be derived (\cf~\cite{liu2020generative}).

\subsection{Regularizing GANs under Limited Data}
\label{sec:method}
Our goal is to improve the performance of the GAN models when the training set $\mathcal{T}$ merely contains a limited amount of data, as the example shown in \figref{teaser}.
Different from the existing data augmentation methods~\cite{ada,diffaug}, we approach this problem by incorporating the regularization on the discriminator.
We present the overview of the proposed method in \figref{overview}.
The core idea is to regulate the discriminator predictions during the training phase.
Specifically, we introduce two exponential moving average~\cite{laine2016temporal} variables $\alpha_R$ and $\alpha_F$, called \emph{anchors}, to track the discriminator's predictions of the real and generated images.
The computation of the anchors $\alpha_R$ and $\alpha_F$ is provided in Eq.~\eqref{eq:ema}.
We then use the identical objective $L_G$ described in Eq.~\eqref{eq:gan} for training the generator, and minimize the regularized objective $L_D$ for the discriminator:
\vspace{-2mm}
\begin{align}
\min_D L_D, L_D = -V_D + \lambda \reg(D),
\label{eq:loss_d}
\vspace{-2mm}
\end{align}
where $\reg$ is the proposed regularization term:
{\small
\begin{align}
\!\reg \!=\! \E_{\bx \sim \D} \! \big[ \|\Dx \!-\! \alpha_F\|^2 \big]
\!+\!\! \E_{\bz \sim p_{\bz}} \! \big[ \|D(G(\bz)) \!-\! \alpha_R\|^2 \big].
\label{eq:reg}
\vspace{-1mm}
\end{align}
}\noindent
At first glance, the objective in Eq.~\eqref{eq:loss_d} appears counterintuitive since the regularization term $\reg$ pushes the discriminator to mix the predictions of real and generated images, as opposed to differentiating them.
However, we show in \secref{lecam} that $\reg$ offers meaningful constraints for optimizing a more robust objective.
Moreover, we empirically demonstrate in \secref{exp} that with the appropriate weight $\lambda$, this simple regularization scheme 1) improves the generalization under limited training data, and 2) complements the existing data augmentation methods.

\Paragraph{Why moving averages?} Tracking the moving average of the prediction reduces the variance across mini-batches and stabilizes the regularization term described in Eq.~\eqref{eq:reg}. 
Intuitively, the moving average becomes stable while the discriminator's prediction gradually converges to the stationary point.
We find this holds for the GAN models used in our experiments (\eg \figref{dcurve}).
%
%
We illustrate a general case of using two moving average variables $\alpha_R$ and $\alpha_F$ in \figref{overview}. 
In some cases, \eg in theoretical analysis, we may use a single moving average variable to track the predictions of either real or generated images.

\subsection{Connection to LeCam Divergence}
\label{sec:lecam}
We show the connection of the proposed regularization to the WGAN~\cite{wgan} model and an $f$-divergence called LeCam (LC)-divergence~\cite{le2012asymptotic} or triangular discrimination~\cite{vincze1981concept}.
Under mild assumptions, our regularization method can enforce WGANs to minimize the weighted LC-divergence.
We show that the LC-divergence 1) can be used for training GAN models robustly under limited training data, and 2) has a close relationship with the $f$-divergences used in other GAN models~\cite{goodfellow2014generative,lsgan,zhao2016energy}.

We first revisit the definition of the $f$-divergence. 
For two discrete distributions $Q(x)$ and $P(x)$, an $f$-divergence is defined as:
\vspace{-2mm}
\begin{equation}
\label{eq:fdivergence}
D_f(P \| Q) = \sum_x Q(x) f(\frac{P(x)}{Q(x)})
\vspace{-2mm}
\end{equation}
if $f$ is a convex function and $f(1)=0$. 
The $f$-divergence plays a crucial role in GANs as it defines the underlying metric to align the generated distribution $\pg$ and data distribution $\pd$. 
For instance, Goodfellow~\etal~~\cite{goodfellow2014generative} showed that the saturated GAN minimizes the JS-divergence~\cite{lin1991divergence} between the two distributions:
\vspace{-1mm}
\begin{align}
\label{eq:cg_goodfellow}
C(G) = 2JS(p_d \| p_g) - \log(4),
\vspace{-1mm}
\end{align}
\noindent
where $C(G)$ is the virtual objective function for the generator when $D$ is fixed to the optimal. 
Similarly, the LSGAN~\cite{lsgan} method leads to minimizing the $\chi^2$-divergence~\cite{pearson1900x} and the EBGAN~\cite{zhao2016energy} scheme minimizes the total variation distance~\cite{wgan}.

More recently, the Wasserstein distance~\cite{wgan}, which does not belong to the $f$-divergence family, introduces a different distribution measurement. 
However, the performance of WGANs and similar models, \eg~BigGAN~\cite{biggan}, deteriorates when the training data is limited.
We show that incorporating the proposed regularization into these GAN models improves the generalization performance, especially under limited training data.
Next, we show the connection between the regularized WGANs and LC-divergence.
\begin{figure}[t]
    \centering
    \includegraphics[width=0.8\linewidth]{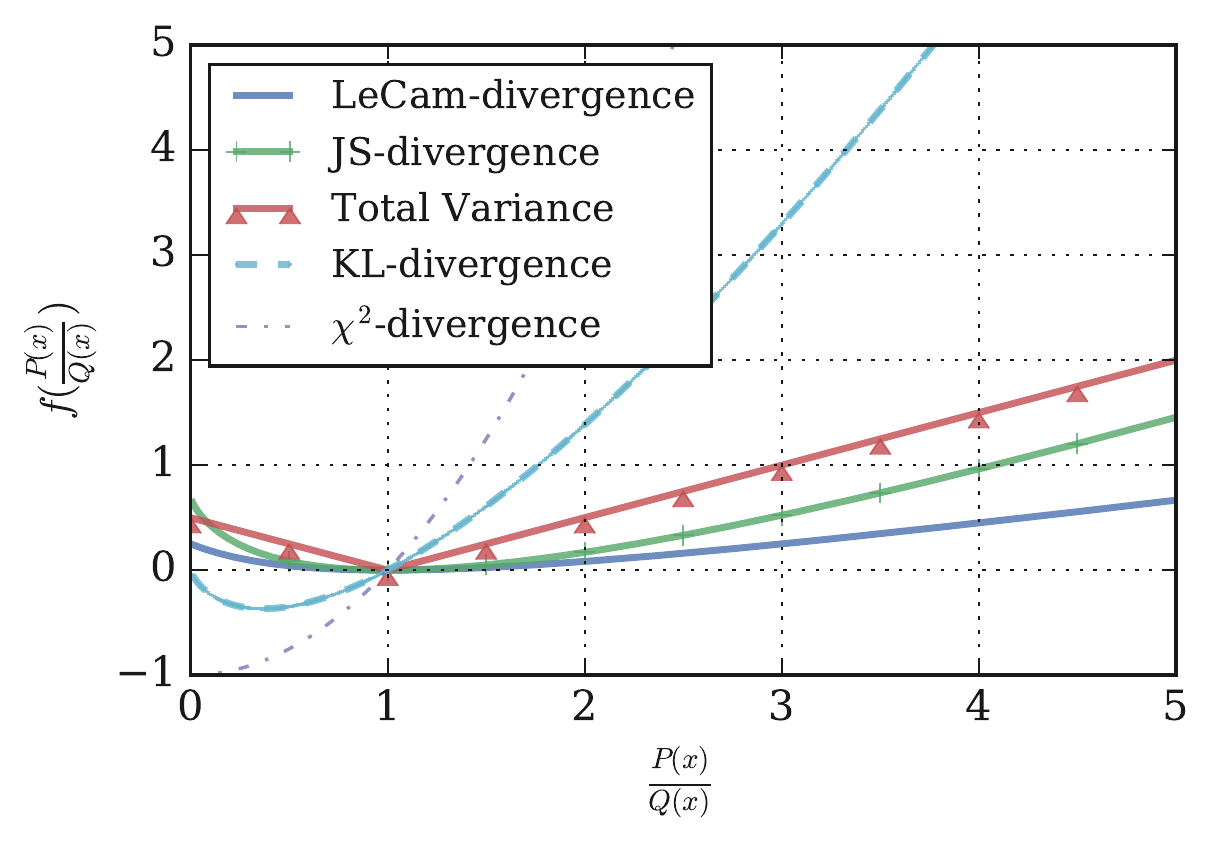}
    \vspace{-2mm}
    \vspace{\figcapmargin}
    \caption{\textbf{Comparison of various $f$-divergences.} The $x$- and $y$-axis denote the input and the value of the function $f$ in the $f$-divergence in Eq.~\eqref{eq:fdivergence}. For extremely large or small inputs of $P(x)/Q(x)$, LeCam-divergence yields the most robust values of $f(P(x)/Q(x))$. The weighted LeCam-divergence is plotted where the weight is $\frac{1}{2\lambda} - \alpha = \frac{1}{4}$.
    }
    \label{fig:lecam}
    \vspace{-3.5mm}
\end{figure}

\vspace{-1mm}
\begin{proposition}
\label{prop:lecam_connection}
Consider the regularized objective in Eq.~\eqref{eq:loss_d}~for the WGAN~\cite{wgan}, where $\reg$ is with a single anchor and $\lambda > 0$.
Assume that with respect to a fixed generator $G$, the anchor converges to a stationary value $\alpha$ ($\alpha > 0$).
Let $C(G)$ denote the virtual objective function of the generator for the fixed optimal $D$. We have:
\vspace{-1.5mm}
\begin{equation}
\label{eq:cg_ours}
C(G) = (\frac{1}{2\lambda}- \alpha) \Delta(p_d \| p_g),
\vspace{-2mm}
\end{equation}
where $\Delta(P \| Q)$ is the LeCam (LC)-divergence aka the triangular discrimination~\cite{le2012asymptotic} given by:
\vspace{-1.5mm}
\begin{equation}
\Delta(P \| Q) = \sum_{x} \frac{(P(x)-Q(x))^2}{(P(x)+Q(x))}.
\vspace{-2mm}
\end{equation}
\end{proposition}

Since the divergence is non-negative, we need $\lambda < \frac{1}{2 \alpha}$, which indicates the regularization weight should not be too large. 
The proof is given in \secref{appendix_proof}. 
We note that the analysis in Proposition~\ref{prop:lecam_connection}, which uses only a single anchor, is a simplified regularizer of our method described in \secref{method}.
To achieve better performance and more general applications, we use 1) two anchors and 2) apply the regularization term to the hinge~\cite{biggan,lim2017geometric} and non-saturated loss~\cite{goodfellow2014generative,stylegan2} in the experiments.
We note this is not a rare practice in the literature.
For example, Goodfellow~\etal~\cite{goodfellow2014generative} show theoretically the saturated GAN loss minimizes the JS-divergence.
However, in practice, they use the non-saturated GAN for superior empirical results.

After drawing the connection between LC-divergence and regularized WGANs, we show that the LC-divergence is a robust $f$-divergence when limited data is available.
\figref{lecam} illustrates several common $f$-divergences, where the $x$-axis plots the input to the function $f$ in Eq.~\eqref{eq:fdivergence}, \ie $P(x)/Q(x)$, and the $y$-axis shows the function value of $f$. 
Note that the input $P(x)/Q(x)$ is expected to be erroneous when limited training data is available, and likely to include extremely large/small values. 
\figref{lecam} shows that the LC-divergence helps obtain a more robust function value for extreme inputs.
In addition, the LC-divergence is symmetric and bounded between $0$ and $2$ which attains the minimum if and only if $p_d=p_g$.
These properties demonstrate the LC-divergence as a robust measurement when limited training data is available. 
This observation is consistent with the experimental results shown in \secref{exp}.

\vspace{-1mm}
\begin{proposition}[Properties of LeCam-divergence]
\label{prop:lecam_properties}
LC-divergence $\Delta$ is an $f$-divergence with following properties: 
\begin{compactitem}
    \item $\Delta$ is non-negative and symmetric.
    \item $\Delta(p_d\|p_g)$ is bounded, with the minimum 0 when $p_d=p_g$ and the maximum $2$ when $p_d$ and $p_g$ are disjoint.
    \item $\Delta$-divergence is a symmetric version of $\chi^2$-divergence, \ie
    $\Delta(P\|Q)= \chi^2(P \| M) + \chi^2(Q \| M)$, where $M=\frac{1}{2} (P+Q)$.
    \item The following inequalities hold~\cite{lu2015class}: $\frac{1}{4} \Delta(P,Q) \le JS(P,Q) \le \frac{1}{2} \Delta(P,Q) \le \frac{1}{2} TV(P, Q)$, where $JS$ and $TV$ represent JS-divergence and Total Variation. 
\end{compactitem}
\end{proposition}

Proposition~\ref{prop:lecam_properties} shows that the LC-divergence is closely related to the $f$-divergences used in other GAN methods. For example,
it is a symmetric and smoothed $\chi^2$-divergence used in the LSGAN~\cite{lsgan}.
The weighted $\Delta$ lower bounds the JS-divergence used in the saturated GAN~\cite{goodfellow2014generative} and the Total Variation distance used in the EBGAN~\cite{zhao2016energy} approaches.

%% file: 4_result.tex
\input{table/cifar}

\vspace{\secmargin}
\section{Experimental Results}
\label{sec:exp}
\vspace{\secmargin}
We conduct extensive experiments on several benchmarks to validate the efficacy of our method on training the leading class-conditional BigGAN~\cite{biggan} and unconditional StyleGAN2~\cite{stylegan2} models on the limited data.

\begin{figure}[t]
    \centering
    \includegraphics[width=0.99\linewidth]{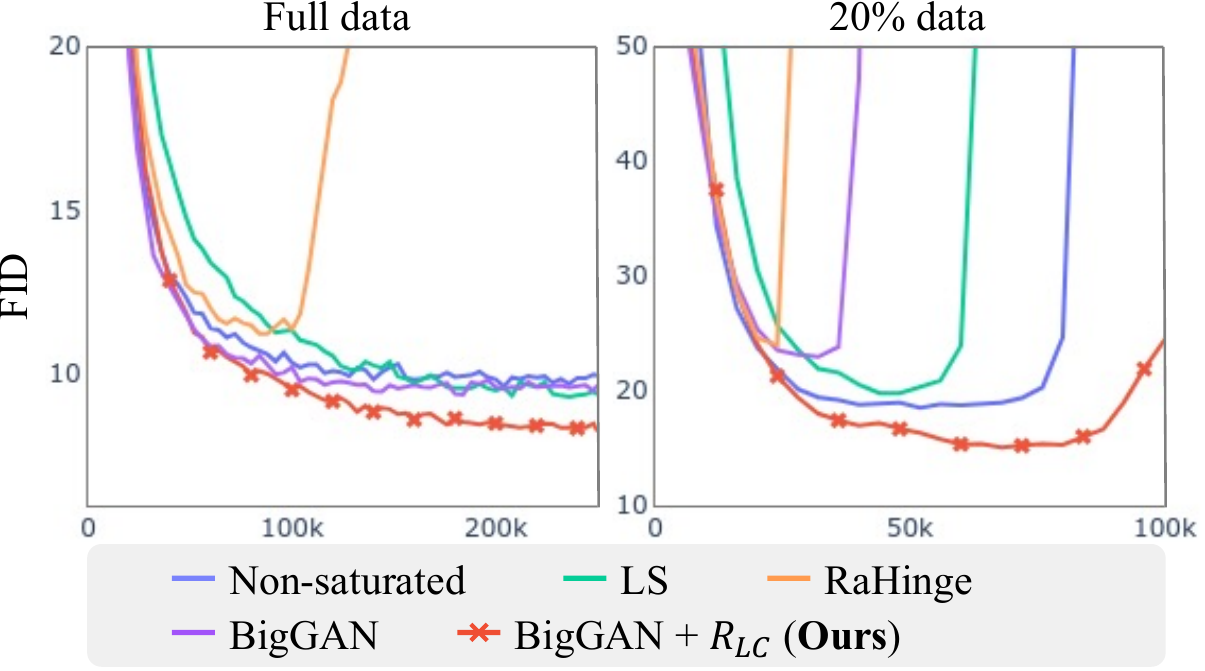}
    \vspace{\figcapmargin}
    \caption{\textbf{FID curves during the training on the CIFAR-10 dataset.} The proposed method 1) improves the best performance, and 2) stabilizes the training dynamic of the BigGAN model under the limited (\eg $20\%$) data setting.}
    \vspace{-1.5mm}
    \label{fig:fid_cifar}
\end{figure}

\Paragraph{Datasets.}
%
The CIFAR 10/100~\cite{cifar} and ImageNet~\cite{imagenet} datasets are standard benchmarks for the image generation models.
The resolutions of the images in the CIFAR, ImageNet datasets are 32x32, and 128x128, respectively.

\Paragraph{Evaluation metrics.}
We use two common metrics: \textit{Inception Score (IS)}~\cite{salimans2016improved} and \textit{Fr\'echet Inception Distance (FID)}~\cite{ttur}.
Unless specified otherwise, we follow the evaluation protocol in the DA paper~\cite{diffaug} that reports the average and standard deviation values over three evaluation trials.

\Paragraph{Setups.}
We conduct the CIFAR experiments using the BigGAN~\cite{biggan} framework implemented by Zhao~\etal~\cite{diffaug}.\footnote{\label{note1}\scriptsize \url{https://github.com/mit-han-lab/data-efficient-gans}}
We train the BigGAN model on TPU for the ImageNet experiments.\footnote{\scriptsize \url{https://github.com/google/compare\_gan}}
%
Finally, the StyleGAN2~\cite{stylegan2} framework is trained and evaluated using the implementation from Zhao~\etal~\cite{diffaug} and Karras~\etal~\cite{ada}.\textsuperscript{\ref{note1}}\footnote{\scriptsize \url{https://github.com/NVlabs/stylegan2-ada}}
As for the hyper-parameter settings, we use the decay factor of $0.99$ for the exponential moving average variables.
%
We set the regularization weight $\lambda$ to $0.3$, $0.01$ for the CIFAR, ImageNet experiments, respectively.

\Paragraph{Baselines.} 
We compare three types of baseline methods on the CIFAR datasets.
The first group are GAN models that optimize various loss functions including \emph{\textbf{non-saturated}}~\cite{goodfellow2014generative}, \emph{\textbf{LS}}~\cite{lsgan}, and \emph{\textbf{RaHinge}}~\cite{relativisticgan}.
Second, we compare with three regularization methods: instance noise~\cite{sonderby2016amortised}, zero-centered gradient penalty (\emph{\textbf{GP-0}})~\cite{gp0} and consistency regularization (\emph{\textbf{CR}})~\cite{zhang2019consistency}.
Finally, we compare with two recent differentiable data augmentation methods \emph{\textbf{DA}}~\cite{diffaug} and \emph{\textbf{ADA}}~\cite{ada} that address the limited data issue for GANs.
%
For the experiments on other datasets, we focus on comparing with the state-of-the-art methods.
For a fair comparison, we compare the baseline methods under the same GAN backbone using their official implementation on each dataset, except \tabref{sota_imagenet} in which we cite the numbers of~\cite{diffaug} reported in the original paper.

\input{table/cifar_gp}
\input{table/sota}
\subsection{Results on CIFAR-10 and CIFAR-100}
\label{sec:exp_1}
As shown in \tabref{cifar}, the proposed method improves the generalization performance of the BigGAN model.
The comparison between other GAN models shows the competitive performance of the proposed method, especially under limited training data.
These results substantiate that our regularization method minimizes a sensible divergence on limited training data.
To further understand the impact on the training dynamics, we plot the FID scores during the training stage in \figref{fid_cifar}.
The proposed method stabilizes the training process on limited data (\ie FID scores deteriorate in a later stage) and achieves the lowest FID score at the final iteration ($100$K).
This result suggests that our method can stabilize the GAN training process on limited data.

\input{table/imagenet}
We compare our method with three regularization methods: instance noise~\cite{sonderby2016amortised}, GP-0~\cite{gp0} and CR~\cite{zhang2019consistency} in \tabref{cifar_gp}.
Notice that the spectral norm regularization~\cite{spectral} is used by default in the BigGAN model~\cite{biggan}.
For the GP-0 method, we apply the gradient penalty only on real images.
Our regularization scheme performs favorably against these regularization methods, particularly under the limited data setting.
%
Despite the improvement under the limited data, the GP-0 approach degrades the FID performance when using the full training data.
We note that a similar observation is raised in the BigGAN paper~\cite{biggan}.
In \figref{gp}, we visualize the GP-0 values of the models trained with the GP-0 and our methods during the training stage.
Interestingly, the proposed method also constrains the GP-0 values, although it does not explicitly minimize the GP-0 loss.
\begin{figure}[t]
    \centering
    \includegraphics[width=0.99\linewidth]{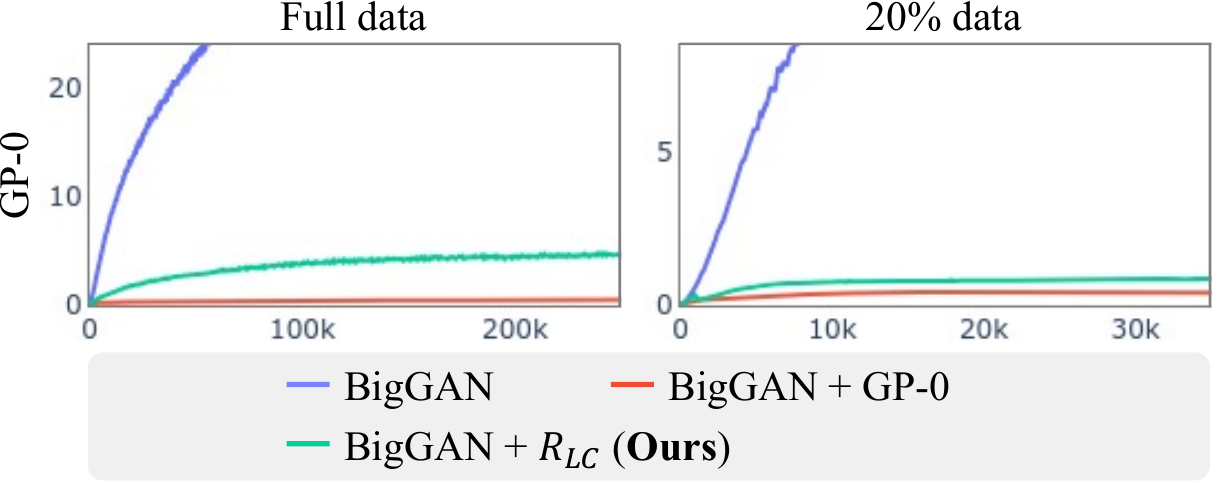}
    \vspace{\figcapmargin}
    \caption{\textbf{Zero-centered gradient penalty values.} We visualize the values of zero-centered gradient penalty (GP-0)~\cite{gp0} during the training stage. The proposed regularization also constrains the values without explicitly minimizing the GP-0 loss.
    }
    \vspace{-1mm}
    \label{fig:gp}
\end{figure}

Finally, we combine our regularization method with data augmentation and show it is complementary to the recent data augmentation methods~\cite{ada,diffaug}.
As presented in \tabref{sota}, the proposed approach improves the performance of DA and ADA, especially under the limited data settings.
Note that the data augmentation methods tackle the problem from different perspectives and represent the prior state-of-the-art on limited training data before this work.

\subsection{Comparison to State-of-the-art on ImageNet}
ImageNet~\cite{imagenet} is a challenging dataset since it contains more categories and images with higher resolution.
Considering the variance of the model performance, we follow the evaluation protocol in the BigGAN paper~\cite{biggan}.
Specifically, we run the training/evaluation pipeline three times using different random seeds, then report the average performance.
We present the quantitative results in \tabref{imagenet}.
The proposed method improves the resistance of the BigGAN model against the scarce training data issue (\eg $\downarrow3.75$ in FID under $25\%$ data).
It is noteworthy that the performance variance of our models is reduced in most cases (\eg $2.59 \rightarrow 1.73$ in FID under $25\%$ data), suggesting its capability in stabilizing the training process.

\tabref{sota_imagenet} demonstrates the quantitative results compared to the state-of-the-art model that uses the DA~\cite{diffaug} method.
Both the quantitative results and qualitative comparison presented in \figref{imagenet} validate that the proposed method complements the data augmentation approach.
We achieve state-of-the-art performance on the limited (\eg $10\%$) ImageNet dataset by combining our regularization and the data augment approaches.
\begin{figure*}[t]
    \centering
    \includegraphics[width=0.95\linewidth]{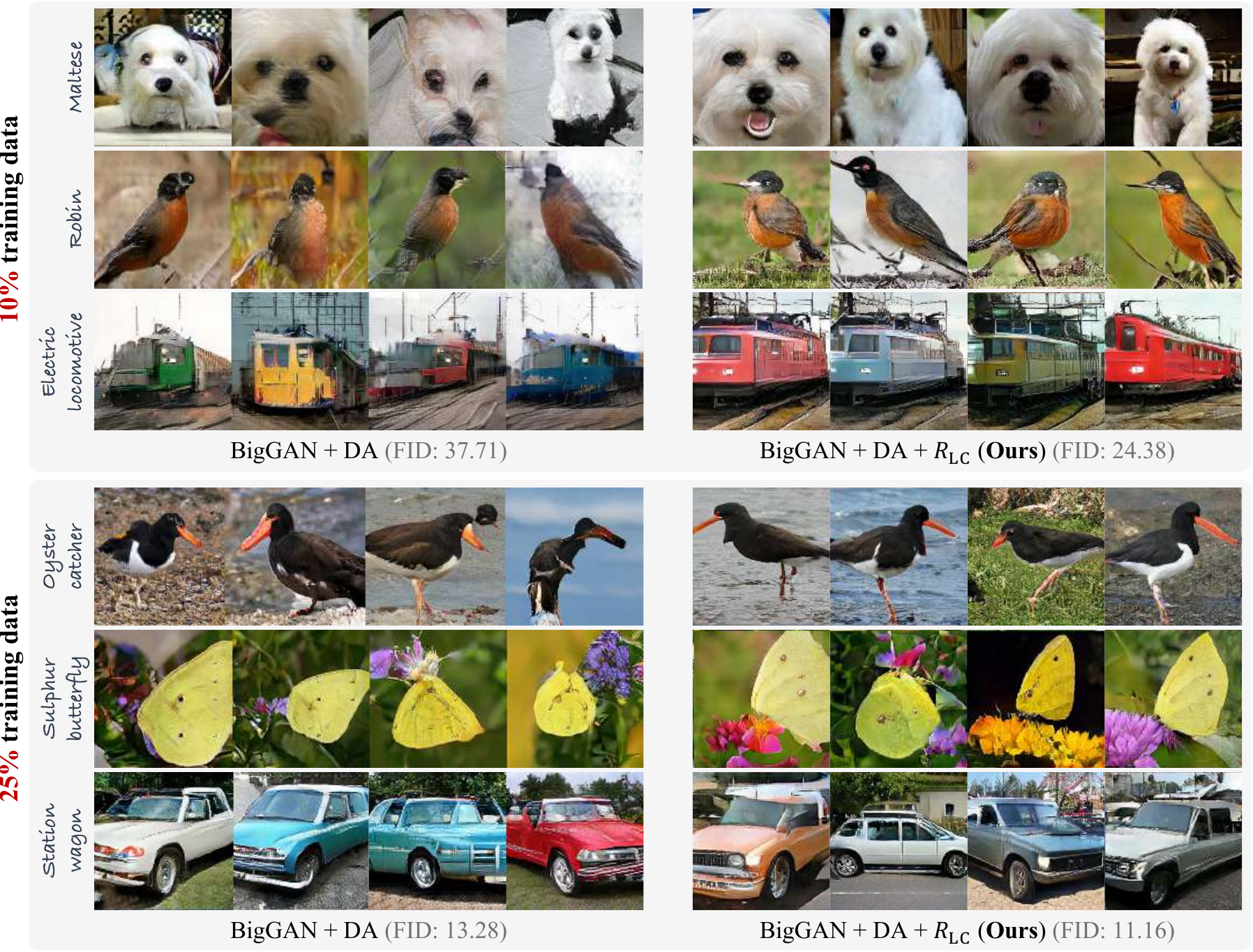}
    \vspace{\figcapmargin}
    \caption{\textbf{Qualitative comparisons under limited training data.} We show the generation results on the (\textit{top}) $10\%$ and (\textit{bottom}) $25\%$ ImageNet dataset. The baseline models trained with our approach synthesize more realistic images.}
    \vspace{-2mm}
    \label{fig:imagenet}
\end{figure*}
\input{table/sota_imagenet}
\input{table/ffhq}

\subsection{Comparison with Data Augmentation}
%
We use the StyleGAN~dataset to conduct the experiments.
Experiment details are provided in \secref{appendix_details}. 
As presented in \tabref{sota} and \tabref{ffhq}, the proposed method improves the performance of the StyleGAN2 model trained with(out) data augmentation~\cite{ada} in all cases.
We note that different from BigGAN, the StyleGAN2 model minimizes the non-saturated~\cite{goodfellow2014generative} GAN loss and uses the gradient penalty GP-0~\cite{gp0} in the default setting.
This shows that the proposed regularization scheme can be applied to other GAN loss functions along with existing regularization approaches.

We make a comparison in \tabref{compare_aug} to summarize the (dis)advantages of the data augmentation and our methods.
First, the data augmentation approaches yield more significant gain than the proposed method when the training data is extremely limited.
Nevertheless, our method can further improve the performance of data augmentation due to the complementary nature of the two methods.
%
Second, the data augmentation approaches may degrade the performance when the training images are sufficiently diverse (\eg the full dataset). This is consistent with the observation described in~\cite{ada}. 
In comparison, our regularization method may not suffer the same problem.
\input{table/compare_aug}

\subsection{Analysis and Ablation Studies}
\vspace{-1mm}
We use the BigGAN model and the CIFAR-10 dataset to conduct the analysis and ablation studies.

\Paragraph{Regularization strength for $R_\mathrm{LC}$.}
We conduct a sensitive study on the regularization weight $\lambda$.
As shown in \figref{modelsize}(b), weights greater than $0.5$ degrade the performance.
This agrees with our analysis in Eq.~\eqref{eq:cg_ours} that larger weights $\lambda$ result in negative divergence values. 
Generally, the proposed method is effective when the weight $\lambda$ is in a reasonable range, \eg $[0.1, 0.5]$ in \figref{modelsize}(b).
%

\Paragraph{Regularizing real \vs generated image predictions.}
Our default method regularizes the predictions of both real images $D(x)$ and generated images $D(G(z))$.
In this experiment, we investigate the effectiveness of separately regularizing the two terms $D(x)$ and $D(G(z))$.
As shown in \tabref{cifar_rf}, regularizing both terms achieves the best result.
\input{table/cifar_rf}

\Paragraph{Discriminator predictions.}
We visualize the discriminator predictions during training in \figref{dcurve}.
Without regularization, the predictions of real and generated images diverge rapidly as the discriminator overfits the limited training data.
On the other hand, the proposed method, as described in Eq.~\eqref{eq:reg}, penalizes the difference between predictions of real and generated images, thus keeping the predictions in a particular range.
This observation empirically substantiates that the discriminator's prediction gradually converges to the stationary point, and so do the moving average variables $\alpha_R$ and $\alpha_F$.
\begin{figure}[t]
    \centering
    \includegraphics[width=0.99\linewidth]{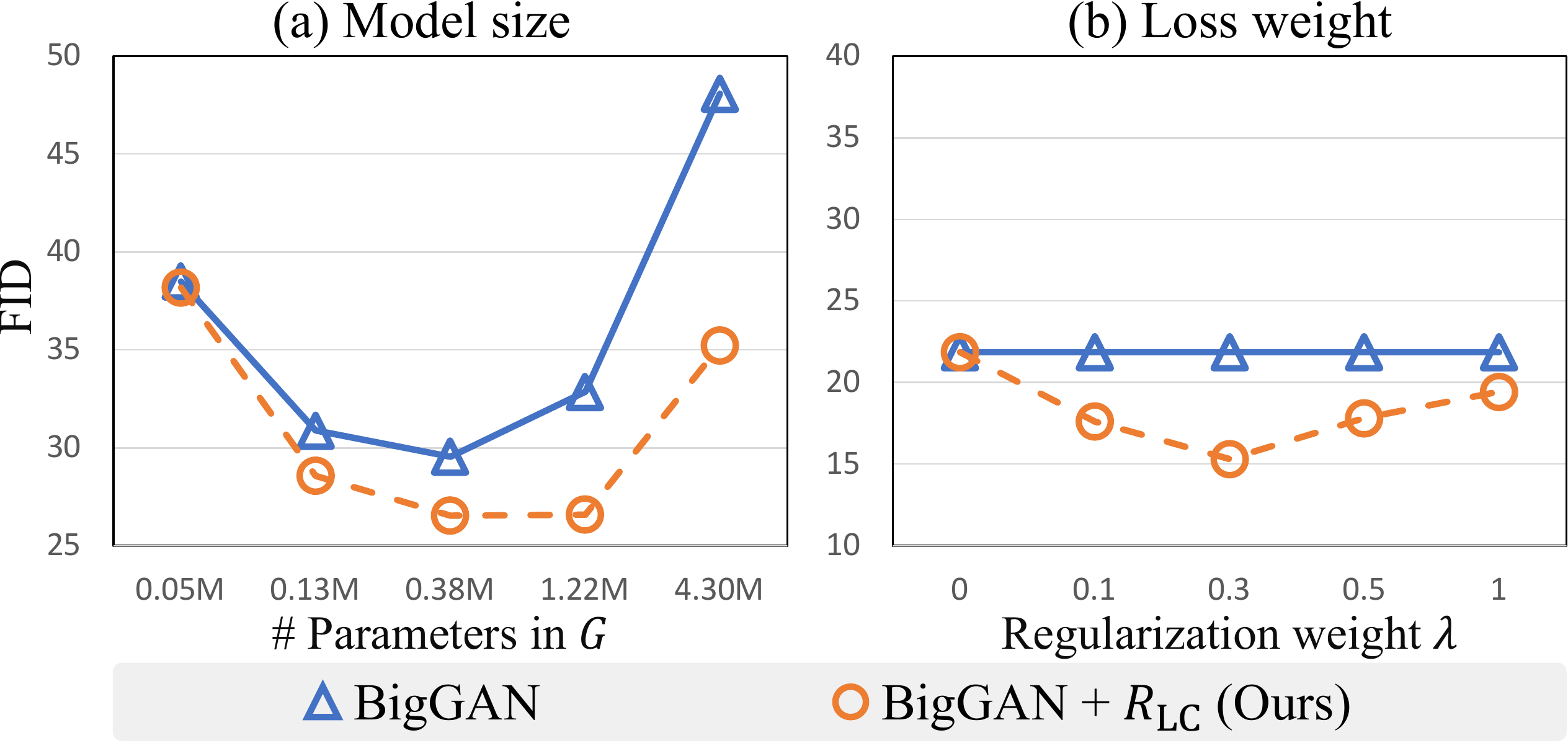}
    \vspace{\figcapmargin}
    \caption{\textbf{Different (a) model sizes and (b) regularization strengths.} The scores are computed on the (a) $10\%$ and (b) $20\%$ CIFAR-10 datasets.}
    \vspace{-2mm}
    \label{fig:modelsize}
\end{figure}

\Paragraph{Model size.}
Since reducing the model capacity may alleviate the overfitting problem, we investigate the performance of using a smaller model size for both generator and discriminator.
\figref{modelsize}(a) shows the results of progressively halving the number of channels in both the generator and discriminator.
The improvement made by our method increases with the model size, as the overfitting issue is more severe for the model with higher capacity.
%

\begin{figure}[t]
    \centering
    \includegraphics[width=0.99\linewidth]{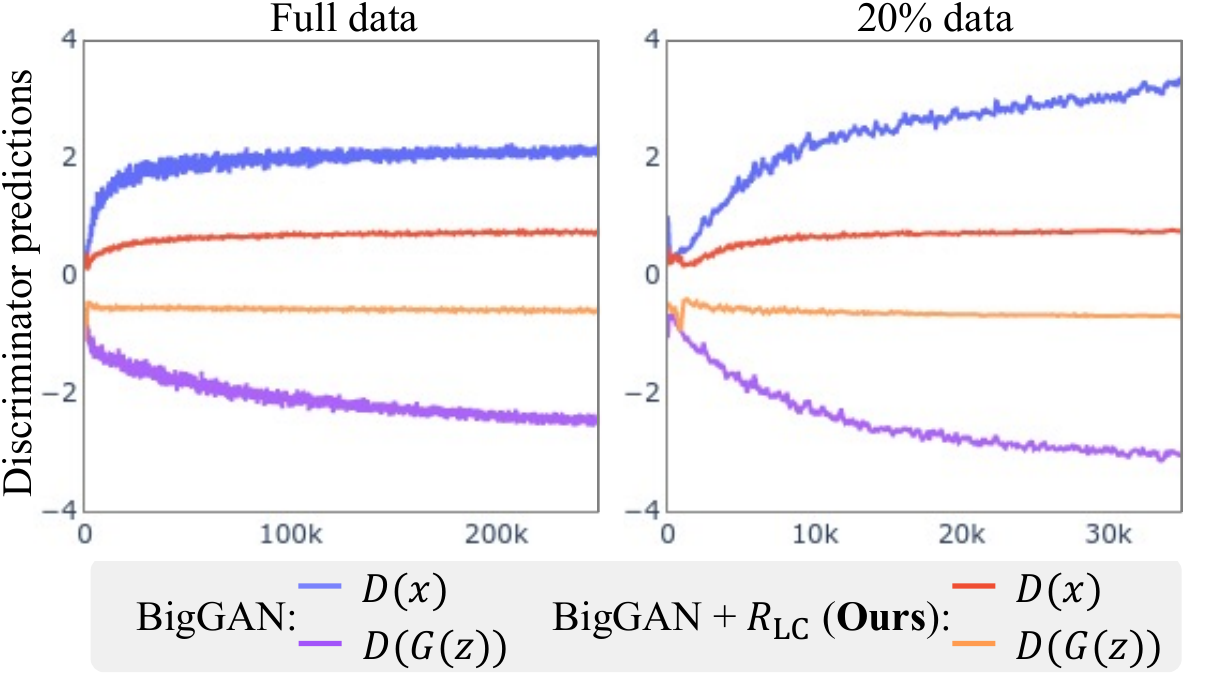}
    \vspace{\figcapmargin}
    \caption{\textbf{Discriminator predictions.} We visualize the discriminator predictions from the BigGAN model on the CIFAR-10 dataset during the training stage. The proposed method prevents the predictions of real images $D(x)$ and generated images $D(G(z))$ from diverging under the limited (\eg $20\%$) data setting.
    }
    \vspace{-2mm}
    \label{fig:dcurve}
\end{figure}

%% file: table/cifar.tex
\begin{table*}[t]
    \caption{\textbf{Quantitative results on the CIFAR dataset.} We report the average FID scores ($\downarrow$) of three evaluation runs. The best performance is in \textbf{bold} and the second best is \underline{underscored}.}
    \vspace{\tabcapmargin}
    \label{tab:cifar}
    \centering
    \footnotesize
    \begin{tabular}{l|ccc ccc} 
	    \toprule
		\multirow{2}{*}{Methods} &  \multicolumn{3}{c}{CIFAR-10} & \multicolumn{3}{c}{CIFAR-100} \\
		\cmidrule(lr){2-4} \cmidrule(lr){5-7} & Full data & $20\%$ data & $10\%$ data & Full data & $20\%$ data & $10\%$ data \\
		\midrule
		Non-saturated GAN~\cite{goodfellow2014generative} & $9.83${\mytiny$\pm0.06$}  & $18.59${\mytiny$\pm0.15$} & $41.99${\mytiny$\pm0.18$} & $13.87${\mytiny$\pm0.08$} & $32.64${\mytiny$\pm0.19$} & $70.50${\mytiny$\pm0.38$}\\
		LS-GAN~\cite{lsgan} & \second{$9.07$}{\mytiny$\pm0.01$} & \second{$21.60$}{\mytiny$\pm0.11$} & \second{$41.68$}{\mytiny$\pm0.18$} & \second{$12.43$}{\mytiny$\pm0.11$} & $\best{27.09}${\mytiny$\pm0.09$} & $54.69${\mytiny$\pm0.12$}\\
		RaHinge GAN~\cite{relativisticgan}& $11.31${\mytiny$\pm0.04$} & $23.90${\mytiny$\pm0.22$} & $48.13${\mytiny$\pm0.33$} & $14.61${\mytiny$\pm0.21$} & \second{$28.79$}{\mytiny$\pm0.17$} & \second{$52.72$}{\mytiny$\pm0.18$}\\
		BigGAN~\cite{biggan} & $9.74${\mytiny$\pm0.06$} & $21.86${\mytiny$\pm0.29$} & $48.08${\mytiny$\pm0.10$} & $13.60${\mytiny$\pm0.07$} & $32.99${\mytiny$\pm0.24$} & $66.71${\mytiny$\pm0.01$} \\
	    BigGAN + $\reg$ (Ours) & $\best{8.31}${\mytiny$\pm0.05$} & $\best{15.27}${\mytiny$\pm0.10$} & $\best{35.23}${\mytiny$\pm0.14$} & $\best{11.88}${\mytiny$\pm0.12$} & $\best{25.51}${\mytiny$\pm0.19$} & $\best{49.63}${\mytiny$\pm0.16$}\\
		\bottomrule
    \end{tabular}
    \vspace{\tabbotmargin}
    \vspace{-1mm}
\end{table*}

%% file: table/cifar_gp.tex
\begin{table}[t]
    \caption{\textbf{Comparison to GAN regularization methods.} We report the average FID ($\downarrow$) scores on the CIFAR datasets.}
    \label{tab:cifar_gp}
    \vspace{\tabcapmargin}
    \centering
    \setlength\tabcolsep{5pt}
    \footnotesize
    \begin{tabular}{@{}l|cccc@{}} 
	    \toprule
		\multirow{2}{*}{Method} & \multicolumn{2}{c}{CIFAR-10} & \multicolumn{2}{c}{CIFAR-100} \\
		\cmidrule(lr){2-3} \cmidrule(lr){4-5} & Full data & $20\%$ data & Full data & $20\%$ data \\
		\midrule
		BigGAN~\cite{biggan} & $9.74${\mytiny$\pm.06$} & $21.86${\mytiny$\pm.29$} & $13.60${\mytiny$\pm.07$} & $32.99${\mytiny$\pm.24$} \\
		+ noise~\cite{sonderby2016amortised} & $9.64${\mytiny$\pm.06$} & $21.87${\mytiny$\pm.11$} & $13.88${\mytiny$\pm.07$} & $32.38${\mytiny$\pm.01$} \\
		+ CR~\cite{zhang2019consistency} & $8.96${\mytiny$\pm.10$} & $20.62${\mytiny$\pm.10$} & $\best{11.59}${\mytiny$\pm.05$} & $36.91${\mytiny$\pm.12$} \\
		+ GP-0~\cite{gp0} & $10.30${\mytiny$\pm.16$} & $19.10${\mytiny$\pm.08$} & $14.67${\mytiny$\pm.08$} & $29.85${\mytiny$\pm.04$} \\
		+ $R_\mathrm{LC}$ (Ours) & $\best{8.31}${\mytiny$\pm.05$} & $\best{15.27}${\mytiny$\pm.10$} & $11.88${\mytiny$\pm.12$} & $\best{25.51}${\mytiny$\pm0.19$}\\
		\bottomrule
    \end{tabular}
    \vspace{\tabbotmargin}
    \vspace{-1mm}
\end{table}

%% file: table/sota.tex
\begin{table*}[t]
    \caption{\textbf{Quantitative comparisons to data augmentation.} We report the average FID ($\downarrow$) scores of three evaluation runs.}
    \vspace{\tabcapmargin}
    \label{tab:sota}
    \centering
    \footnotesize
    \begin{tabular}{l|cccccc} 
	    \toprule
	    \multirow{2}{*}{Methods} & \multicolumn{2}{c}{CIFAR-10} & \multicolumn{2}{c}{CIFAR-100} & \multicolumn{2}{c}{StyleGAN} \\
	    \cmidrule(lr){2-3} \cmidrule(lr){4-5} \cmidrule(lr){6-7} 
	    & Full & $10\%$ & Full & $10\%$ & Full & 1K \\
	    \midrule
	    BigGAN~\cite{biggan} + DA~\cite{diffaug} & $8.75${\mytiny$\pm0.05$} & $23.34${\mytiny$\pm0.28$} & $11.99${\mytiny$\pm0.10$} & $35.39${\mytiny$\pm0.16$} & - & - \\
	    BigGAN + DA + $\reg$ (Ours) & $\best{8.46}${\mytiny$\pm0.06$} & $\best{16.69}${\mytiny$\pm0.02$} & $\best{11.20}${\mytiny$\pm0.09$} & $\best{27.28}${\mytiny$\pm0.05$} & - & - \\
	    \midrule
	    StyleGAN2~\cite{stylegan2} + ADA~\cite{ada} & $2.68${\mytiny$\pm0.02$} & $6.72${\mytiny$\pm0.03$} & $3.04${\mytiny$\pm0.02$} & $14.06${\mytiny$\pm0.07$} & $3.82${\mytiny$\pm0.01$} & $23.27${\mytiny$\pm0.14$} \\
	    StyleGAN2 + ADA+ $\reg$ (Ours)& $\best{2.47}${\mytiny$\pm0.01$} & $\best{6.56}${\mytiny$\pm0.02$} & $\best{2.99}${\mytiny$\pm0.01$} & $\best{13.01}${\mytiny$\pm0.02$} & $\best{3.49}${\mytiny$\pm0.04$} & $\best{21.70}${\mytiny$\pm0.06$} \\
		\bottomrule
    \end{tabular}
    \vspace{\tabbotmargin}
    \vspace{-1mm}
\end{table*}

%% file: table/imagenet.tex
\begin{table*}[t]
    \caption{\textbf{Quantitative results on the ImageNet dataset}. We report the mean IS ($\uparrow$) and FID ($\downarrow$) scores of three training runs.}
    \vspace{\tabcapmargin}
    \label{tab:imagenet}
    \centering
    \footnotesize
    \begin{tabular}{l|cc cc cc } 
	    \toprule
        \multirow{2}{*}{Methods} & \multicolumn{2}{c}{Full data}  & \multicolumn{2}{c}{$50\%$ data}  & \multicolumn{2}{c}{$25\%$ data} \\ 
		\cmidrule(lr){2-3} \cmidrule(lr){4-5} \cmidrule(lr){6-7}
		& IS $\uparrow$ & FID $\downarrow$& IS $\uparrow$ & FID $\downarrow$ & IS $\uparrow$ & FID $\downarrow$ \\
		\midrule
		BigGAN~\cite{biggan} & $90.48${\mytiny $\pm12.7$} & $8.60${\mytiny$\pm1.08$} & $80.26${\mytiny$\pm5.55$} & $9.83${\mytiny$\pm0.94$} & $61.05${\mytiny$\pm6.43$} & $18.22${\mytiny$\pm2.59$} \\
	    BigGAN + $\reg$ (Ours)& $\best{93.00}${\mytiny$\pm3.27$} & $\best{7.27}${\mytiny$\pm0.14$} & $\best{89.94}${\mytiny$\pm6.67$} & $\best{9.13}${\mytiny$\pm0.84$} & $\best{65.66}${\mytiny$\pm4.96$} & $\best{14.47}${\mytiny$\pm1.73$} \\
		\bottomrule
    \end{tabular}
    \vspace{\tabbotmargin}
    \vspace{-1.5mm}
\end{table*}

%% file: table/sota_imagenet.tex
\begin{table*}[t]
    \caption{\textbf{Comparison to the state-of-the-art on the limited ImageNet training data.}. We train and evaluate the BigGAN~\cite{biggan} model following the same evaluation protocol in~\cite{diffaug}. $\dagger$ denotes the result is quoted from~\cite{diffaug}.
    }
    \vspace{\tabcapmargin}
    \label{tab:sota_imagenet}
    \centering
    \footnotesize
    \begin{tabular}{l|cc cc cc cc} 
	    \toprule
		\multirow{2}{*}{Methods} & \multicolumn{2}{c}{Full data} & \multicolumn{2}{c}{$50\%$ data} & \multicolumn{2}{c}{$25\%$ data} & \multicolumn{2}{c}{$10\%$ data} \\
		\cmidrule(lr){2-3} \cmidrule(lr){4-5} \cmidrule{6-7} \cmidrule{8-9} & IS $\uparrow$ & FID $\downarrow$& IS $\uparrow$ & FID $\downarrow$ & IS $\uparrow$ & FID $\downarrow$ & IS $\uparrow$ & FID $\downarrow$ \\
		\midrule
        DA~\cite{diffaug} (Zhao~\etal) & $100.8${\mytiny$\pm0.2^\dagger$} & $6.80${\mytiny$\pm0.02^\dagger$} & $\best{91.9}${\mytiny$\pm0.5^\dagger$} & $8.88${\mytiny$\pm0.06^\dagger$} & $74.2${\mytiny$\pm0.5^\dagger$} &  $13.28${\mytiny$\pm0.07^\dagger$} & $27.7${\mytiny$\pm0.1$} & $37.71${\mytiny$\pm0.11$}\\
        DA + $\reg$ (Ours) & $\best{108.0}${\mytiny$\pm0.6$} & $\best{6.54}${\mytiny$\pm0.03$} & $91.7${\mytiny$\pm0.6$} & $\best{8.59}${\mytiny$\pm0.01$} & $\best{84.7}${\mytiny$\pm0.5$} & $\best{11.16}${\mytiny$\pm0.05$} & $\best{42.3}${\mytiny$\pm0.3$}  & $\best{24.38}${\mytiny$\pm0.06$} \\
		\bottomrule
    \end{tabular}
    \vspace{\tabbotmargin}
    \vspace{-1mm}
\end{table*}

%% file: table/ffhq.tex
\begin{table*}[t]
    \caption{\textbf{Quantitative results of the StyleGAN2~\cite{stylegan2} model.} We report the average FID ($\downarrow$) scores of three evaluation runs.}
    \vspace{\tabcapmargin}
    \label{tab:ffhq}
    \centering
    \footnotesize
    \begin{tabular}{l|ccccc} 
	    \toprule
	    Method & $70$k images & $30$k images & $10$k images & $5$k images & $1$k images \\
	    \midrule
		StyleGAN2~\cite{stylegan2} & $3.79${\mytiny$\pm0.02$} & $6.19${\mytiny$\pm0.05$} & $14.96${\mytiny$\pm0.05$} & $25.88${\mytiny$\pm0.09$} & $72.07${\mytiny$\pm0.04$}\\
	    StyleGAN2 + $\reg$ (Ours) & $\best{3.66}${\mytiny$\pm0.02$} & $\best{5.78}${\mytiny$\pm0.03$} & $\best{14.58}${\mytiny$\pm0.04$} & $\best{23.83}${\mytiny$\pm0.11$} & $\best{63.16}${\mytiny$\pm0.11$} \\
		\bottomrule
    \end{tabular}
    \vspace{\tabbotmargin}
    \vspace{-1mm}
\end{table*}

%% file: table/compare_aug.tex
\begin{table}[t]
    \caption{\textbf{Comparisons with data augmentation methods.} We report the FID ($\downarrow$) scores of the StyleGAN2 backbone.}
    \vspace{\tabcapmargin}
    \label{tab:compare_aug}
    \centering
    \footnotesize
    \begin{tabular}{l|cc} 
	    \toprule
		Method & Full data & 1K data \\
		\midrule
		StyleGAN2~\cite{stylegan2} & $3.71${\mytiny$\pm0.01$} & $72.07${\mytiny$\pm0.04$} \\
		\midrule
		+ DA~\cite{diffaug} & $4.21${\mytiny$\pm0.03$} & $25.17${\mytiny$\pm0.09$} \\
		+ ADA~\cite{ada} & $3.81${\mytiny$\pm0.01$} & $23.27${\mytiny$\pm0.14$} \\
		\midrule
		+$R_\mathrm{LC}$ & $3.66${\mytiny$\pm0.02$} & $63.16${\mytiny$\pm0.11$} \\
		+ ADA + $R_\mathrm{LC}$ & $\best{3.49}${\mytiny$\pm0.04$} & $\best{21.70}${\mytiny$\pm0.06$} \\
		\bottomrule
    \end{tabular}
    \vspace{\tabbotmargin}
    \vspace{-1mm}
\end{table}

%% file: table/cifar_rf.tex
\begin{table}[t]
    \caption{\textbf{Ablation study on regularizing real \vs generated image predictions.} We train and evaluate the BigGAN~\cite{biggan} model on the CIFAR-10 dataset, then report the average FID ($\downarrow$) scores.}
    \vspace{\tabcapmargin}
    \label{tab:cifar_rf}
    \centering
    \footnotesize
    \begin{tabular}{cc|cc}
         \toprule
         Real & Generated & Full data & $20\%$ data \\
         \midrule
         & & $9.74${\mytiny$\pm0.06$} & $21.86${\mytiny$\pm0.29$} \\
         \midrule
		 \checkmark & & $8.73${\mytiny$\pm0.04$} & $20.47${\mytiny$\pm0.36$} \\
		 & \checkmark & $8.79${\mytiny$\pm0.09$} & $18.18${\mytiny$\pm0.08$} \\
		 \checkmark & \checkmark & $\best{8.31}${\mytiny$\pm0.03$} & $\best{15.27}${\mytiny$\pm0.10$} \\
		 \bottomrule
    \end{tabular}
    \vspace{\tabbotmargin}
\end{table}

%% file: 5_conclusion.tex
\vspace{\secmargin}
\section{Conclusion and Future Work}
\vspace{\secmargin}

In this work, we present a regularization method to train the GAN models under the limited data setting.
The proposed method achieves a more robust training objective for the GAN models by imposing a regularization loss to the discriminator during the training stage.
In the experiments, we conduct experiments on various image generation datasets with different GAN backbones to demonstrate the efficacy of the proposed scheme that 1) improves the performance of the GAN models, especially under the limited data setting and 2) can be applied along with the data augmentation methods to further enhance the performance.
%
In future, we plan the training data scarcity issue for 1) the conditional GAN tasks such as image extrapolation, image-to-image translation, \etc, and 2) the robust GAN learning on large-scale noisy training data.

\vspace{\secmargin}
\section*{Acknowledgements}
\vspace{\secmargin}
We would like to thank anonymous reviewers for their useful comments. This work is supported in part by the NSF CAREER Grant \#1149783.

%% file: 6_appendix.tex
\section{Supplementary Materials}

\subsection{Overview}
In this supplementary document, we first provide the theoretical justification for Proposition 1 in the paper.
Second, we describe the implementation details.
Finally, we present additional experimental results, including those of training the GAN model on only hundreds of images, \ie low-shot image generation.

\subsection{Theoretical Analysis}
\label{sec:appendix_proof}
\begin{proposition}
\label{prop:appendix_lecam_connection}
Consider the regularized objective in Eq.(1) and (2) in the paper for the WGAN~\cite{wgan}, where $\reg$ is with a single anchor and $\lambda > 0$. Assume that with respect to a fixed generator $G$, the anchor converges to a stationary value $\alpha$ ($\alpha > 0$). Let $C(G)$ denote the generator's virtual objective for the fixed optimal $D$. We have:
\begin{align}
C(G) = (\frac{1}{2\lambda}- \alpha) \Delta(p_d \| p_g),
\end{align} where $\Delta(P \| Q)$ is the LeCam-divergence aka the triangular discrimination~\cite{le2012asymptotic} given by:
\begin{align}
\Delta(P \| Q) = \sum_{x} \frac{(P(x)-Q(x))^2}{(P(x)+Q(x))}
\end{align}
\end{proposition}

\begin{proof}
In the following, we use $\pd$ to denote the target distribution and simplify $\E_{\bz \sim \pz} \big[ D(G(\bz)) \big]$ using $\E_{\bx \sim \pg} \big[ \Dx \big]$. With a single anchor, the proposed regularization has the following form:
\begin{align}
\reg(D) = \E_{\bx \sim \pd}  \big[ \|\Dx + \alpha\|^2 \big] + \E_{\bx \sim \pg} \big[ \|D(\bx) - \alpha\|^2 \big],
\label{eq:reg_singleanchor}
\end{align}
where $\alpha \ge 0$ is the anchor for the real images, \ie $\alpha_R$ in the Equation (4) in the paper. Note that since $D(G(\bz)) \le 0$, when using a single anchor we have that $\alpha_R = - \alpha_F = \alpha$.

Consider the regularized objective of the discriminator:
\begin{align}
\min L(D) &= \min \E_{\bx \sim \pg} \big[ \Dx \big] - \E_{\bx \sim \pd} \big[ \Dx \big] + \lambda \reg(D) \\
&=\min \E_{\bx \sim \pg} \big[ \Dx \big] - \E_{\bx \sim \pd} \big[ \Dx \big] + \lambda \E_{\bx \sim \pd} \big[ \|\Dx + \alpha\|^2 \big] + \lambda \E_{\bx \sim \pg} \big[ \|\Dx - \alpha\|^2 \big] \\
&=\min \E_{\bx \sim \pd} \big[ \lambda\|\Dx + \alpha\|^2 - \Dx \big] + \E_{\bx \sim \pg} \big[ \lambda \|\Dx - \alpha\|^2 + \Dx \big]\\
&=\min \E_{\bx \sim \pd} \big[ \lambda\|\Dx + \alpha\|^2 - \Dx -\alpha + \frac{1}{4\lambda}\big] + \E_{\bx \sim \pg} \big[ \lambda \|\Dx - \alpha\|^2 + \Dx -\alpha + \frac{1}{4\lambda} \big]  + C\\
&=\min \lambda \E_{\bx \sim \pd} \big[\|\Dx+\alpha-\frac{1}{2\lambda}\|^2 \big] + \lambda \E_{\bx \sim \pg} \big[ \|\Dx+\frac{1}{2\lambda}-\alpha\|^2 \big] + C
\end{align}
where $C=2\alpha -\frac{1}{2\lambda}$.

We now derive the optimal discriminator $D^*$ with respect to a fixed $G$. According to the assumption, near convergence of $D^*$, $C$ approaches a constant value.
This is a mild assumption because we found that the discriminator predictions always converge to the stationary points in all of the experiments for both the WGAN and BigGAN models (\cf Figure 8 in the main paper).
Hypothetically speaking, in rare cases where this criterion might not hold, we may anneal the decay factor in the moving average $\alpha$ gradually to 1.0 near while $D$ approaches convergence. In the following, we treat $C$ as a constant value and compute $D^*$ from:
\begin{align}
\Dx^* = \argmin_D L(D) = \lambda \int_{\bx} \big[  \pd(\Dx+\alpha-\frac{1}{2\lambda})^2 + \pg (\Dx-\alpha+\frac{1}{2\lambda})^2 \big] dx \\
\frac{d L(D)}{d x} = 2 \lambda \big[ \pd (\Dx+\alpha-\frac{1}{2\lambda}) + \pg (\Dx-\alpha+\frac{1}{2\lambda}) \big] =0\\
\implies (\pd+\pg)\Dx + (\pd-\pg) (\alpha-\frac{1}{2\lambda}) =0 \\
\implies \Dsx = \frac{(\pd-\pg)(\frac{1}{2\lambda}-\alpha)}{\pd+\pg}
\end{align}

Consider the following generator's objective when $D$ is fixed:
\begin{align}
\min_G L(G) &= - \E_{\bx \sim \pg} \big[ \Dx \big] + \E_{\bx \sim \pd} \big[ \Dx \big]
\end{align}
Notice that as the regularization term is only added to the discriminator, and the generator's objective is kept the same.
Then we have:
\begin{align}
C(G) &=  \int_{x} \big[ \pd \Dsx -\pg \Dsx \big] dx\\
&= (\frac{1}{2\lambda}-\alpha) \int_{x} \frac{(\pd-\pg)^2}{\pd+\pg} dx\\
&= (\frac{1}{2\lambda}-\alpha) \Delta(\pd \| \pg) \label{eq:lecam},
\end{align}
where $\Delta$ is the LeCam divergence and $\frac{1}{2\lambda}-\alpha$ is the weight of the divergence. Since the divergence is non-negative, we need $\lambda < \frac{1}{2 \alpha}$. For example, if $\alpha=1$, then $\lambda < 0.5$.
This indicates the weight $\lambda$ in the proposed regularization term $\reg$ should not be too large.

The proof is then completed.
\end{proof}

\begin{figure}[t]
    \centering
    \includegraphics[width=0.5\linewidth]{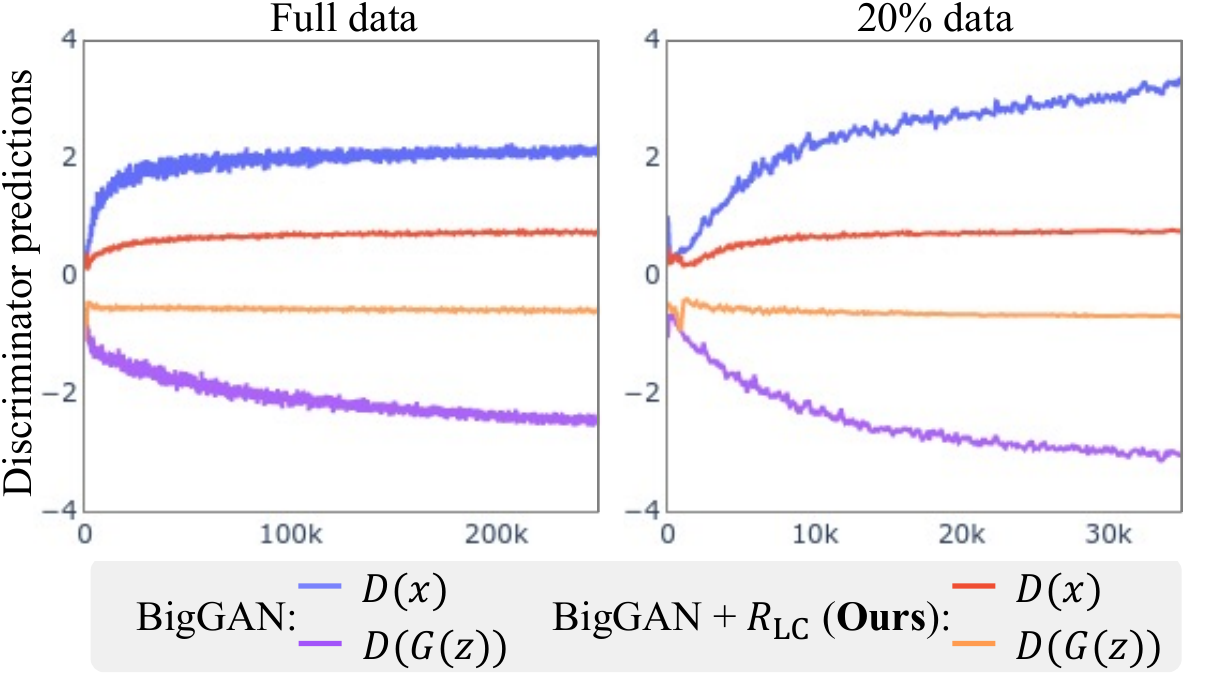}
    \vspace{\figcapmargin}
    \caption{\textbf{Discriminator predictions.} We visualize the discriminator predictions from the BigGAN model on the CIFAR-10 dataset during the training stage. The proposed method prevents the predictions of real images $D(x)$ and generated images $D(G(z))$ from diverging under the limited (\eg $20\%$) data setting.
    }
    \vspace{-2mm}
    \label{fig:appendix_dcurve}
\end{figure}

\Paragraph{Discussion on the theoretical results.}
Our method is inspired by the theoretical analysis in Proposition~\ref{prop:appendix_lecam_connection}.
In our experiments, we employ modifications to optimize the performance.
We note this is not a rare practice in the literature.
For example, Goodfellow~\etal~\cite{goodfellow2014generative} show theoretically the saturated GAN loss minimizes the JS-divergence.
However, in practice, they use the non-saturated GAN due to the superior empirical performance. 

Specifically, our method incorporates two modifications.
First, it uses two anchors for the discriminator predictions of both real and generated images.
Our ablation study in Table 6 in the main paper shows this leads to a performance gain.
Second, we extend our method to regularize other GAN losses in the leading-performing GAN models such as the BigGAN~\cite{biggan} and StyleGAN2~\cite{stylegan2} models.
The former has a similar objective as the WGAN that applies the hinge loss~\cite{lim2017geometric}.
Using a similar procedure in~\cite{wgan}, we might be able to extend the result in Proposition~\ref{prop:appendix_lecam_connection} when the discriminator predictions are within the margin boundaries.

We empirically substantiate the analysis by showing the proposed regularization prevents the discriminator predictions from diverging on the limited training data.
As shown in \figref{appendix_dcurve}, without regularization, the predictions of real and generated images diverge rapidly under the limited data setting. On the other hand, the proposed method keeps the predictions within -1 and +1.

\subsection{Implementation Details}
\label{sec:appendix_details}

\Paragraph{Exponential moving average.}
We implement the exponential moving average operation using the following formulation:
\begin{align}
    \label{eq:ema}
    \alpha^{(t)} = \gamma \times \alpha^{(t-1)} + (1-\gamma) \times v^{(t)},
\end{align}
where $\alpha$ is the moving average variable (\ie $\alpha_R$ and $\alpha_F$), $v^{(t)}$ is the current value at training step $t$, and $\gamma$ is the decay factor.
We fix the decay factor $\gamma$ to $0.99$ in all experiments.

\Paragraph{CIFAR-10 and CIFAR-100.} 
We set the weight $\lambda$ of regularization term to $0.3$, and adopt the default hyper-parameters of the baseline method in the implementation by Zhao~\etal~\cite{diffaug}.\footnote{\url{https://github.com/mit-han-lab/data-efficient-gans/tree/master/DiffAugment-biggan-cifar}}
Specifically, we use the batch size of $50$, learning rate of $2e-4$ for the generator $G$ and discriminator $D$, $4$ $D$ update steps per $G$ step, and \textit{translation + cutout} for the DA~\cite{diffaug} method.

\Paragraph{ImageNet.}
We use the Compare GAN codebase\footnote{\url{https://github.com/google/compare_gan}} for the experiments on the ImageNet dataset.
The random scaling, random horizontal flipping operations are used to pre-process the images.
We keep the default hyper-parameter settings for the baseline methods (\ie BigGAN~\cite{biggan}, BigGAN + DA~\cite{diffaug}).
As for our approach, we use the batch size of $2048$, learning rate of $4e{-4}$ for $D$ and $1e{-4}$ for $G$, $2$ $D$ update steps per $G$ step, and the regularization weight $\lambda$ of $0.01$.

\Paragraph{Comparisons with Data Augmentation}
We train and evaluate the StyleGAN2~\cite{stylegan2} framework on the FFHQ~\cite{stylegan} dataset, where the image size is $256\times256$.
We set the regularization weight $\lambda$ to $3e-7$ in this experiment.
%
We use the ADA~\cite{ada} codebase\footnote{\url{https://github.com/NVlabs/stylegan2-ada}} and the DA~\cite{diffaug} source code\footnote{\label{supp_note4}\url{https://github.com/mit-han-lab/data-efficient-gans/tree/master/DiffAugment-stylegan2}} for the experiments shown in Table 3 and Table 6 in the paper, respectively.
Since the StyleGAN2 model uses the softplus mapping function for computing the GAN loss, the gradients of the discriminiator predictions around zero are much smaller than those in the BigGAN~\cite{biggan} model that uses the hinge function i.e., BigGAN: $0.8$, StyleGAN2: $10^{-3}$ in the last layer of the discriminator).
Therefore, we use a much smaller regularization weight $\lambda$ of $3e-7$.
Though the weight $\lambda$ is smaller on the FFHQ dataset, we can observe the impact of our method by comparing StyleGAN2 ($\lambda$=$0$) and StyleGAN2+$R_{LC}$ ($\lambda$=$3e-7$) in Table 3, 6 and 7 in the paper.
Moreover, the ablation study results in Fig 7(b) suggest our approach is relatively insensitive to the value of $\lambda$ under the same backbone.
As for the other hyper-parameters, we keep the setting used in the original implementations.

\input{supp_table/cifar_fid}
\input{supp_table/cifar_is}

\Paragraph{Reproducing results of previous methods.}
We obtain quantitatively comparable results in most experiments.
However, there are few cases that we fail to reproduce the results reported in the original paper.
First, compared to Table 3 in the DA~\cite{diffaug} paper, we obtain different results of training the StyleGAN2 model on the $5$k and $1$k \FFHQ~datasets, respectively.
Second, the result of training the StyleGAN model with the ADA~\cite{ada} method on the $1$k \FFHQ~dataset is slightly different from that reported in 7(c) in the ADA paper.
On the other hand, the result of training the StyleGAN2 model on the full \FFHQ~dataset is similar to that shown in the DA and ADA papers.
As a result, we argue that the different sets of limited data sampled for training the StyleGAN model (using the different random seeds) cause the performance discrepancy observed under the limited data setting.

\begin{table}[t]
    \caption{\textbf{Ablation study on exponential moving averages (EMAs).} We validate the impact of the EMAs by replacing the EMAs with the constant values. We train and evaluate the BigGAN model on the CIFAR dataset in this experiment.}
    \label{tab:ablation_ema}
    \centering
    \scriptsize
    \begin{tabular}{@{}c|c|cc@{}}
         \toprule
         FID($\downarrow$) & EMAs & [$\alpha_R$=$0.5$, $\alpha_F$=$-0.5$] & [$\alpha_R$=$1$, $\alpha_F$=$-1$] \\
         \midrule
         CIFAR 10 & $\best{15.27}${\mytiny$\pm0.10$} & $30.64${\mytiny$\pm0.05$} & $19.81${\mytiny$\pm0.03$}\\
         CIFAR 100 & $\best{25.51}${\mytiny$\pm0.19$} & $30.03${\mytiny$\pm0.11$} & $27.54${\mytiny$\pm0.07$}\\
		 \bottomrule
    \end{tabular}
\end{table}
\begin{figure}[t]
    \centering
    \includegraphics[width=0.95\linewidth]{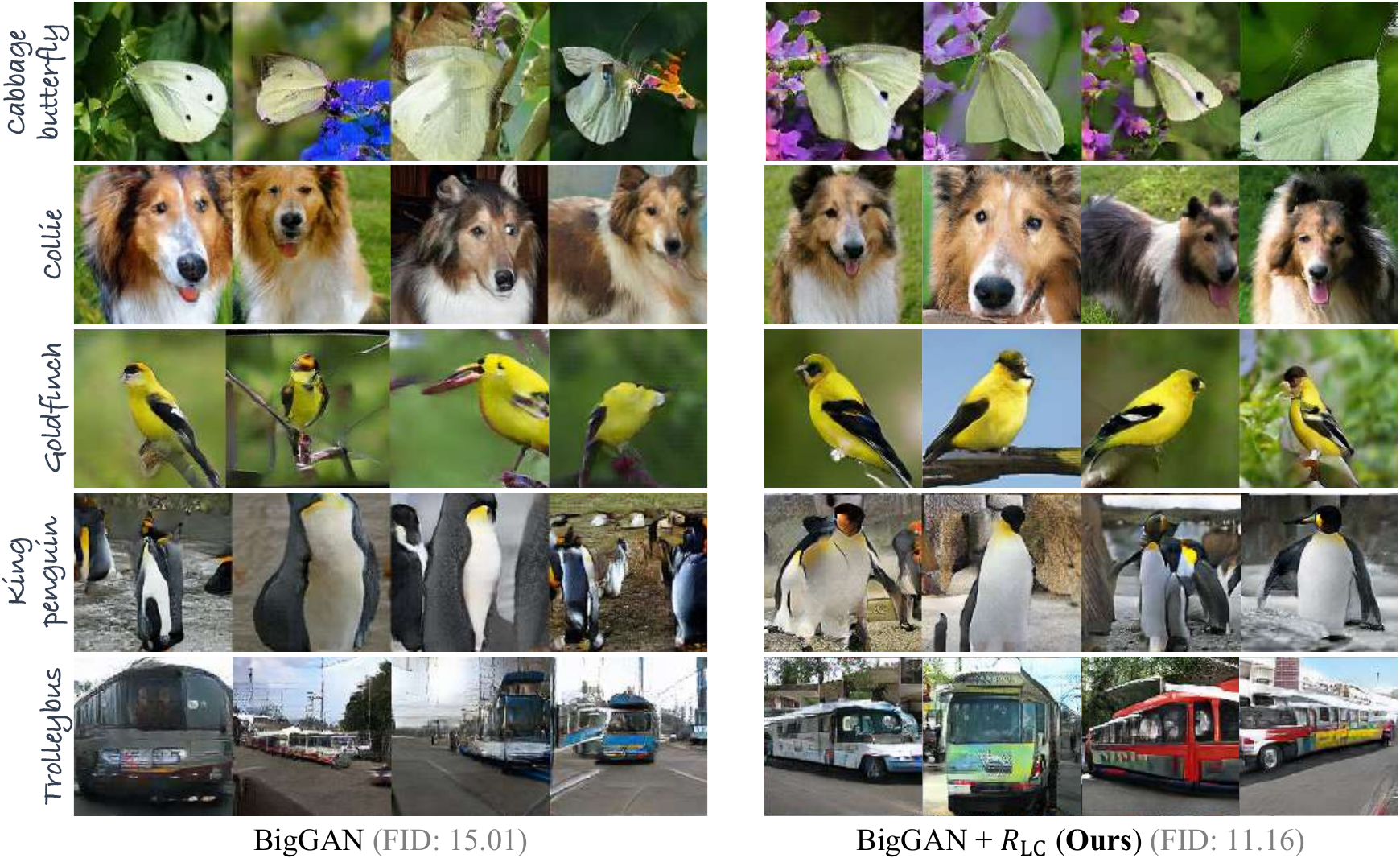}
    \vspace{\figcapmargin}
    \caption{\textbf{Qualitative comparisons under limited training data.} We show the generation results on the $25\%$ ImageNet dataset. The baseline models trained with our approach synthesize more realistic images.}
    \vspace{-2mm}
    \label{fig:supp_imagenet_1}
\end{figure}

\subsection{Additional Experimental Results}

\subsubsection{CIFAR-10 and CIFAR-100}
We report the results of WGAN on the CIFAR-10 dataset in \tabref{supp_cifar_fid}.
Although the proposed method is able to improve the performance of the WGAN model, the performance of the WGAN backbone is inferior to that of the BigGAN backbone and is also more sensitive to the hyperparameter setting.
Therefore, we use the BigGAN backbone in our CIFAR and ImageNet experiments. 
In addition, \tabref{supp_cifar_is} presents the IS scores to complement the FIS scores reported in Table 1 in the paper for the CIFAR experiments.

\Paragraph{Necessity of exponential moving averages (EMAs).}
We validate the necessity of the EMAs in the table below with the BigGAN model on the $20\%$ CIFAR datasets.
Specifically, we compute our regularization with constant anchors by setting 1) $\alpha_R$=1 and $\alpha_F$=$-1$ following the LS-GAN [43] 2) $\alpha_F$=$-0.5$ and $\alpha_R$=$0.5$ (\figref{appendix_dcurve} shows $\pm$ 0.5 is similar to the converged value of EMAs.)
The results in \tabref{ablation_ema} show that using EMAs empirically facilitates the discriminator to converge to the better local optimal.

\subsubsection{ImageNet}
We show additional qualitative comparisons between the baseline (\ie BigGAN~\cite{biggan}) and the proposed method (\ie BigGAN + $\reg$) in \figref{supp_imagenet_1}.
Combining the qualitative results shown in Figure 6 in the paper, we find that the proposed approach improves the visual quality of the generated images compared to the baseline models with and without data augmentation.

\input{supp_table/100shot.tex}
\begin{figure}[t]
    \centering
    \includegraphics[width=0.95\linewidth]{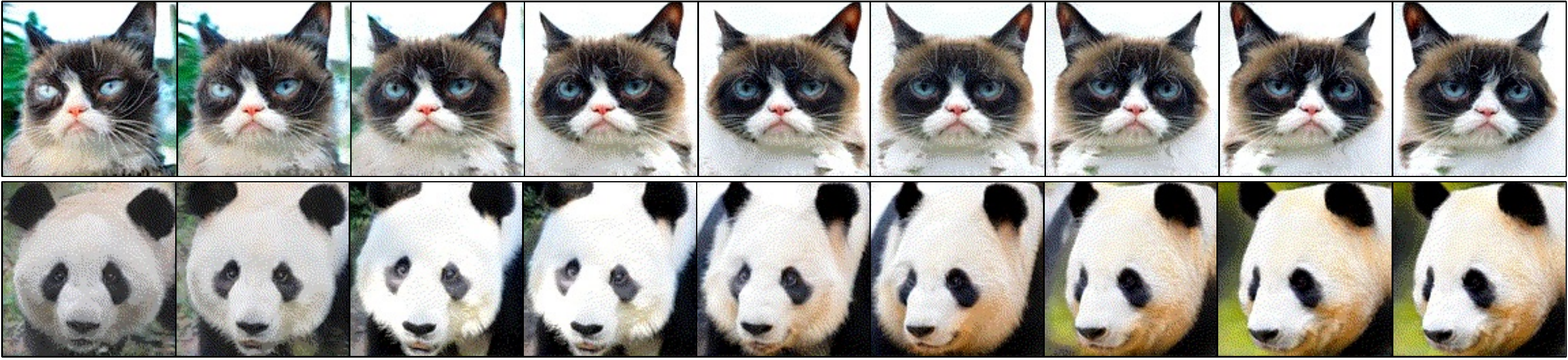}
    \vspace{\figcapmargin}
    \caption{\textbf{Low-shot generation results.} We train the StyleGAN2 model with the proposed regularization and data augmentation methods on the $100$ grumpy cat (\textit{top}), Obama (\textit{middle}), and panda (\textit{bottom}) images.}
    \vspace{-2mm}
    \label{fig:lowshot}
\end{figure}

\subsubsection{Low-Shot Image Generation}
In this experiment, we consider a more extreme scenario where only a few dozens of images are available for training a GAN model.
This setting is known as the \emph{low-shot image generation} problem~\cite{diffaug}.
Recent solutions focus on adapting an exiting GAN model pre-trained on other large datasets.
The adaptation strategies include optimizing the whole GAN model~\cite{transfergan}, modifying the batch statics~\cite{noguchi2019image}, using an additional mining network~\cite{minegan}, and fine-tuning parts of the GAN model~\cite{freezeD}.
We use the experimental setting in the DA~\cite{diffaug} paper that trains and evaluates the StyleGAN2 model on datasets that contain only $100$ (Obama, Grumpy cat, Panda), $160$ (Cat), or $389$ (Dog) images.\textsuperscript{\ref{supp_note4}}
We set the regularization weight $\lambda$ to $0.0001$.
The quantitative comparisons are shown in \tabref{lowshot}.
The StyleGAN2 model trained with the proposed regularization and data augmentation methods \emph{from scratch} performs favorably against the existing adaptation-based techniques.
Note that the adaptation-based approaches require to pre-train the StyleGAN2 model on the \FFHQ~dataset consisting of $70000$ images.
We also perform the interpolation in the latent space, and present the image generation results in \figref{lowshot}.

%% file: supp_table/cifar_fid.tex
\begin{table*}[t]
    \caption{\textbf{Comparisons to WGAN on the CIFAR-10 dataset.} We report the average FID ($\downarrow$) scores of three evaluation runs.}
    \vspace{\tabcapmargin}
    \label{tab:supp_cifar_fid}
    \centering
    \small
    \begin{tabular}{l|cc cc cc cc} 
	    \toprule
		\multirow{2}{*}{Methods} & \multicolumn{2}{c}{WGAN~\cite{wgan}} & \multicolumn{2}{c}{WGAN + $R_\mathrm{LC}$ (Ours)} & \multicolumn{2}{c}{BigGAN~\cite{biggan}} & \multicolumn{2}{c}{BigGAN + $R_\mathrm{LC}$ (Ours)} \\
		\cmidrule(lr){2-3} \cmidrule(lr){4-5} \cmidrule(lr){6-7} \cmidrule(lr){8-9} & IS ($\uparrow$) & FID ($\downarrow$) & IS ($\uparrow$) & FID ($\downarrow$) & IS ($\uparrow$) & FID ($\downarrow$) & IS ($\uparrow$) & FID ($\downarrow$) \\
		\midrule
		Full CIFAR-10 & $7.86${\mytiny$\pm.07$} & $18.86${\mytiny$\pm.13$} & $7.98${\mytiny$\pm.02$} & $15.79${\mytiny$\pm.11$} & $9.07${\mytiny$\pm0.03$} & $9.74${\mytiny$\pm0.06$} & $\best{9.31}${\mytiny$\pm0.04$} & $\best{8.31}${\mytiny$\pm0.05$}\\
		\bottomrule
    \end{tabular}
    \vspace{\tabbotmargin}
    \vspace{-1mm}
\end{table*}

%% file: supp_table/cifar_is.tex
\begin{table}[t]
    \caption{\textbf{IS scores on the CIFAR dataset.} We report the average IS scores ($\uparrow$) of three evaluation runs to supplement Table 1 in the paper. The best performance is in \textbf{bold} and the second best is \underline{underscored}.}
    \vspace{\tabcapmargin}
    \label{tab:supp_cifar_is}
    \centering
    \small
    \begin{tabular}{l|ccc ccc} 
	    \toprule
		\multirow{2}{*}{Methods} &  \multicolumn{3}{c}{CIFAR-10} & \multicolumn{3}{c}{CIFAR-100} \\
		\cmidrule(lr){2-4} \cmidrule(lr){5-7} & Full data & $20\%$ data & $10\%$ data & Full data & $20\%$ data & $10\%$ data \\
		\midrule
		Non-saturated GAN~\cite{goodfellow2014generative} & \second{$9.08$}{\mytiny$\pm0.11$} & $8.36${\mytiny$\pm0.09$} & \second{$7.80$}{\mytiny$\pm0.07$} & $10.58${\mytiny$\pm0.13$} & $8.75${\mytiny$\pm0.07$} & $5.96${\mytiny$\pm0.05$}\\
		LS-GAN~\cite{lsgan} & $9.05${\mytiny$\pm0.10$} & $8.50${\mytiny$\pm0.08$} & $7.33${\mytiny$\pm0.08$} & \second{$10.75$}{\mytiny$\pm0.08$} & \second{$8.94$}{\mytiny$\pm0.01$} & \second{$7.02$}{\mytiny$\pm0.11$}\\
		RaHinge GAN~\cite{relativisticgan}& $8.96${\mytiny$\pm0.05$} & \second{$8.52$}{\mytiny$\pm0.04$}  & $6.84${\mytiny$\pm0.04$} & $10.46${\mytiny$\pm0.12$} & $9.19${\mytiny$\pm0.08$} & $6.95${\mytiny$\pm0.07$} \\
		BigGAN~\cite{biggan} & $9.07${\mytiny$\pm0.03$} & \second{$8.52$}{\mytiny$\pm0.10$} & $7.09${\mytiny$\pm0.03$} & $10.71${\mytiny$\pm0.14$} & $8.58${\mytiny$\pm0.04$} & $6.74${\mytiny$\pm0.04$}\\
	    BigGAN + $\reg$ (Ours) & $\best{9.31}${\mytiny$\pm0.04$} & $\best{8.78}${\mytiny$\pm0.07$} & $\best{7.97}${\mytiny$\pm0.03$} & $\best{10.95}${\mytiny$\pm0.07$} & $\best{9.63}${\mytiny$\pm0.06$} & $\best{7.76}${\mytiny$\pm0.01$} \\
		\bottomrule
    \end{tabular}
    \vspace{\tabbotmargin}
    \vspace{-1mm}
\end{table}

%% file: supp_table/100shot.tex
\begin{table}[t]
    \caption{\textbf{Quatitative results on the low-shot image generation datasets.} We report the average FID scores ($\downarrow$) of three evaluation runs. The best performance is \textbf{bold} and the second best is \underline{underscored}. Using the proposed regularization approach along with data augmentation to train the model on only \emph{100} (Obama, Grumpy cat, Panda), \emph{160} (Cat), or \emph{389} (Dog) images perform favorably against the transfer learning techniques that pre-train the model on \emph{70000} images.}
    \vspace{\tabcapmargin}
    \label{tab:lowshot}
    \centering
    \small
    \begin{tabular}{lc|ccc cc} 
	    \toprule
	    \multirow{2}{*}{Methods} & \multirow{2}{*}{Pre-training?} & \multicolumn{3}{c}{100-shot~\cite{diffaug}} & \multicolumn{2}{c}{AnimalFace~\cite{animalface}} \\
		\cmidrule(lr){3-5} \cmidrule(lr){6-7} & & Obama & Grumby cat & Panda & Cat & Dog \\
		\midrule
		Scale/shift~\cite{noguchi2019image} & \checkmark & $50.72$ & $ 34.20$ & $21.38$ & $54.83$ & $83.04$ \\
		MineGAN~\cite{minegan} & \checkmark & $50.63$ & $34.54$ & $14.84$ & $54.45$ & $93.03$\\
		TransferGAN~\cite{transfergan} & \checkmark & $48.73$ & $34.06$ & $23.20$ & $52.61$ & $82.38$ \\
		TransferGAN +DA~\cite{diffaug} & \checkmark & $39.85$ & $29.77$ & $17.12$ & $49.10$ & $65.57$ \\
		FreezeD~\cite{freezeD} & \checkmark & $41.87$ & $31.22$ & $17.95$ & $47.70$ & $70.46$ \\
		TransferGAN +DA & \checkmark & \second{$35.75$} & $29.32$ & $14.50$ & $46.07$ & $61.03$ \\
		\midrule
		StyleGAN2~\cite{stylegan2} & & $80.45${\mytiny$\pm.36$} & $48.63${\mytiny$\pm.05$} & $34.07${\mytiny$\pm.22$} & $69.84${\mytiny$\pm.19$} & $129.9${\mytiny$\pm.03$} \\
		StyleGAN2 + DA & & $47.09${\mytiny$\pm.14$} & \second{$27.21$}{\mytiny$\pm.03$} & \second{$12.13$}{\mytiny$\pm.07$} & \second{$42.40$}{\mytiny$\pm.07$} & \second{$58.47$}{\mytiny$\pm.06$} \\
		StyleGAN2 + DA + $R_\mathrm{LC}$ (\textbf{Ours}) & & $\best{33.16}${\mytiny$\pm.23$} &  $\best{24.93}${\mytiny$\pm.12$} & $\best{10.16}${\mytiny$\pm.05$} & $\best{34.18}${\mytiny$\pm.11$} & $\best{54.88}${\mytiny$\pm.09$} \\
		\bottomrule
    \end{tabular}
    \vspace{\tabbotmargin}
    \vspace{-1mm}
\end{table}